\documentclass[letterpaper, 11pt]{article}

\usepackage{microtype}
\usepackage{graphicx}
\usepackage{natbib}
\usepackage{subcaption}
\usepackage{booktabs} % for professional tables

\usepackage{hyperref}

\usepackage[margin=1in]{geometry}

\usepackage{amsmath}
\usepackage{amssymb}
\usepackage{mathtools}
\usepackage{amsthm}
\usepackage{soul}

\usepackage[capitalize,noabbrev]{cleveref}

\theoremstyle{plain}
\newtheorem{theorem}{Theorem}[section]

\newtheorem{lemma}[theorem]{Lemma}
\newtheorem{claim}[theorem]{Claim}

\theoremstyle{definition}
\newtheorem{definition}[theorem]{Definition}

\theoremstyle{remark}

\usepackage[textsize=tiny]{todonotes}

\crefname{theorem}{Theorem}{Theorems}
\crefname{lemma}{Lemma}{Lemmas}
\crefname{claim}{Claim}{Claims}
\crefname{proposition}{Proposition}{Propositions}
\crefname{fact}{Fact}{Facts}
\crefname{corollary}{Corollary}{Corollaries}
\crefname{assumption}{Assumption}{Assumptions}
\crefname{definition}{Definition}{Definitions}
\crefname{remark}{Remark}{Remarks}
\crefname{observation}{Observation}{Observations}
\crefname{notation}{Notation}{Notations}
\crefname{example}{Example}{Examples}
\crefname{examples}{Examples}{Examples}
\crefname{question}{Question}{Questions}
\crefname{inequality}{Inequality}{Inequalities}
\crefname{table}{Table}{Tables}
\crefname{section}{Section}{Sections}
\crefname{algorithm}{Algorithm}{Algorithms}
\crefname{equation}{Equation}{Equations}
\crefname{appendix}{Appendix}{Appendices}

\usepackage{algorithm}
\usepackage[noend]{algpseudocode}

\renewcommand{\Pr}{\operatorname{{Pr}}}

\newcommand{\cX}{\mathcal{X}}
\newcommand{\cA}{\mathcal{A}}
\newcommand{\cI}{\mathcal{I}}

\newcommand{\eps}{\varepsilon}
\renewcommand{\epsilon}{\eps}

\hypersetup{
    colorlinks=true,
    linkcolor=blue,
    filecolor=magenta,      
    urlcolor=teal,
    citecolor=blue,
}

\newcommand{\R}{\mathbb{R}} % Real numbers
\tikzset{
  block/.style={draw, rounded corners=3pt, align=center, inner sep=7pt, font=\small},
  smallblock/.style={draw, rounded corners=3pt, align=left, inner sep=6pt, font=\footnotesize},
  arrow/.style={-{Stealth[length=6pt,width=6pt]},very thick},
  note/.style={font=\footnotesize, align=left}
}

\title{Fast-MWEM: Private Data Release in Sublinear Time}

\begin{document}

\author{
  Themistoklis Haris \\ Boston University
  \and
  Steve Choi \\ Boston University
  \and
  Mutiraj Laksanawisit \\ Boston University
}

\maketitle

\begin{abstract}
The Multiplicative Weights Exponential Mechanism (MWEM) is a fundamental iterative framework for private data analysis, with broad applications such as answering $m$ linear queries, or privately solving systems of $m$ linear constraints. However, a critical bottleneck hindering its scalability is the $\Theta(m)$ time complexity required to execute the exponential mechanism in each iteration. We introduce a modification to the MWEM framework that improves the per-iteration runtime dependency to $\Theta(\sqrt{m})$ in expectation. This is done via a lazy sampling approach to the Report-Noisy-Max mechanism, which we implement efficiently using Gumbel noise and a $k$-Nearest Neighbor data structure. This allows for the rapid selection of the approximate score in the exponential mechanism without an exhaustive linear scan. We apply our accelerated framework to the problems of private linear query release and solving Linear Programs (LPs) under neighboring constraint conditions and low-sensitivity assumptions. Experimental evaluation confirms that our method provides a substantial runtime improvement over classic MWEM.
\end{abstract}

\section{Introduction}
Differential Privacy (DP) has emerged as the gold standard for private data analysis, providing a rigorous mathematical framework for preventing the leakage of personal information \citep{dwork2006calibrating}. The core principle of DP is that the outcome of a randomized algorithm should not significantly change if a single individual's data is added to or removed from a dataset. Since its introduction, this definition has led to the development of a wide variety of differentially private algorithms that are often deployed in industry and government \citep{dwork2019differential}.

The Multiplicative Weights Update (MWU) method is a powerful algorithmic tool connecting diverse fields such as online learning, game theory, and optimization \citep{hardt2012simple}. In its classic formulation, MWU is an iterative algorithm for decision-making under uncertainty. We consider a sequential game where a player must choose an action from $n$ options at each step, and an adversary reveals a corresponding loss vector $\ell^{(t)} \in \mathbb{R}^n$. MWU provides a strategy for updating a probability distribution over the actions by systematically down-weighting actions with high loss. In this manner, MWU achieves remarkably low regret against the best fixed action in hindsight.

Privatizing MWU has yielded numerous important results in private data analysis. In the \textbf{Multiplicative Weights Exponential Mechanism (MWEM)}, the MWU strategy is augmented with a private selection mechanism. At each step, the algorithm uses the private exponential mechanism \citep{mcsherry2007mechanismdesign} to select a high-quality action. This powerful combination allows for the private solving of complex tasks, including linear query answering \citep{hardt2012simple}, solving linear programs \citep{hsu2014privately}, boosting weak learners \citep{bun2020efficient}, and private synthetic data generation.

However, despite its versatility and strong theoretical guarantees, MWEM suffers from a critical performance bottleneck. The use of the exponential mechanism in each iteration requires evaluating the quality of every possible action, leading to a runtime that scales \textit{linearly} with the number of queries or constraints. For large-scale problems with millions of data points and thousands of queries, this linear dependency renders the algorithm computationally intractable. While various heuristics can reduce the dependency on the domain size, the linear scaling with the number of queries remains a fundamental barrier.

Our work addresses this bottleneck by rethinking the implementation of MWEM's core private oracle. We introduce a method for executing the exponential mechanism that sidesteps its costly exhaustive search, thereby reducing the per-iteration runtime to be sublinear in the number of queries. This significant acceleration comes with minimal loss to the provable privacy and utility guarantees.

\subsection{Our Results and Techniques}

Our primary contribution is \textbf{Fast-MWEM}, an algorithmic framework that reduces the per-iteration time complexity of the exponential mechanism from $\Theta(m)$ to an expected $\Theta(\sqrt{m})$. This acceleration effectively breaks the linear barrier that has traditionally hindered MWEM's scalability in high-dimensional settings.

We demonstrate the versatility of Fast-MWEM by applying it to two fundamental problems in private data analysis:

\paragraph{Private Linear Query Release.}
We first address the classic problem of answering a collection $Q$ of $m$ linear queries on a dataset $X$ consisting of $n$ elements from a domain set $\mathcal{X}$. Viewing the dataset as a normalized histogram vector $h \in [0,1]^{|\mathcal{X}|}$ and each query as a vector $q \in [0,1]^{|\mathcal{X}|}$, our goal is to output a ``synthetic histogram'' vector $\hat{p}$ that minimizes the maximum error:
\begin{align}
    \label{eq:lq-form}
    \max_{q \in Q} \left| \langle q, h-\hat{p} \rangle \right|.
\end{align}
We provide an algorithm that matches the utility and privacy guarantees of the classic MWEM algorithm but whose expected per-iteration runtime scales with $\tilde{O}(|\mathcal{X}|\sqrt{m})$, compared to the previous $O(|\mathcal{X}|m)$. For settings with large query sets this yields a substantial speedup.

\paragraph{Private Linear Programming.}
We also apply our framework to the problem of privately solving Linear Programs (LPs) with $m$ constraints. We focus on LPs of the form
\begin{align}
    \label{eq:lp-form}
    \max_{x \in \mathbb{R}^d}\{c^\top x\},\,\,\text{subject to}\,\, Ax\leq b,
\end{align}
with $c\in\mathbb{R}^d, A\in \mathbb{R}^{m\times d}, b \in \mathbb{R}^m$. We consider two settings \citep{hsu2014privately} for neighboring databases $D\sim D'$:
\begin{itemize}
    \item \textit{scalar-private, low-sensitivity} LPs (\Cref{sec:primal}): we have $A(D)=A(D')$, $c(D)=c(D')$ and $||b(D)-b(D')||_\infty \leq \Delta_\infty$.
    \item \textit{constraint-private} LPs (\Cref{sec:dual}): the constraints $(A(D),b(D))$ and $(A(D'),b(D'))$ differ in that one of them has one extra row.
\end{itemize}

\textbf{Fast-MWEM} maintains the utility guarantees of \cite{hsu2014privately} while reducing expected per-iteration runtime from $O(dm)$ to $O(d\sqrt{m})$ for scalar-private LPs, and by $O(m\sqrt{d})$ for large-width constraint-private LPs.

\paragraph{Our Techniques.}
The core technical innovation driving these results is a novel implementation of the exponential mechanism. We observe that in many applications, the "score" of a candidate can be expressed as an inner product.

\begin{itemize}
    \item \textbf{Lazy Gumbel Sampling:} We leverage the ``lazy'' Gumbel sampling technique \citep{mussmann2017fast}, which allows us to sample from the exponential mechanism's distribution by examining only the top-$\sqrt{m}$ scores rather than the full set.
    \item \textbf{Reduction to MIPS:} We show that identifying these top scores is equivalent to the Maximum Inner Product Search (MIPS) problem. This insight allows us to utilize efficient, off-the-shelf $k$-MIPS data structures like Locality Sensitive Hashing (LSH) \citep{datar2004locality}, Inverted File Indices (IVF) \citep{baranchuk2018revisiting} and Hierarhical Navigable Small Worlds graphs (HNSW) \citep{malkov2018efficient} to implement the private selection step in sublinear time.
\end{itemize}

\paragraph{Experimental Evaluation.}
We validate our framework on synthetic benchmarks for both linear query release and private linear programming. Our experiments confirm that Fast-MWEM achieves a significant runtime improvement over standard MWEM, with the performance gap widening as the number of queries or constraints increases. Crucially, this acceleration is achieved without compromising utility; our method maintains competitive error rates and privacy guarantees compared to the baseline exhaustive approach.

\begin{figure}[h]
    \centering
    \includegraphics[scale=0.13]{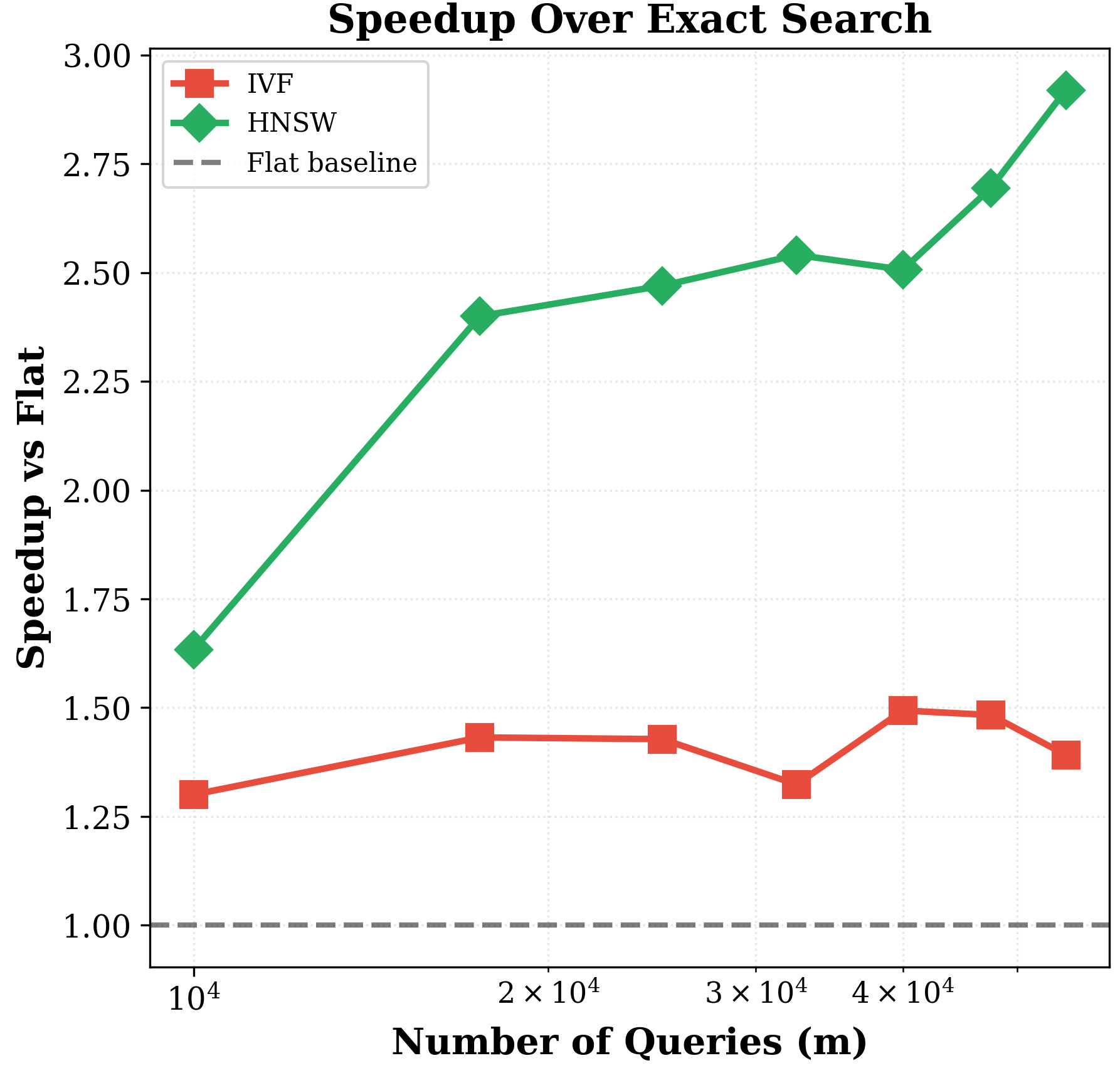}
    \caption{Observed speed-up factor of Fast-MWEM on linear queries over exhaustive search for IVF and HNSW indices.}
    \label{fig:placeholder}
\end{figure}

\subsection{Related Work} % Themis
Since its inception \citep{hardt2012simple}, MWEM and its variants has been used heavily in private data analysis. In the problem of linear query release, numerous works have provided close to optimal error guarantees while also achieving practical performance \citep{liu2021iterative, roth2010interactive, el2020practical}. For private linear programming and optimization, MWEM is a crucial ingredient in many algorithms \citep{hsu2014privately}, though recent developments have revealed more efficient algorithms \citep{ene2025solving, kaplan2025differentially}.

The runtime bottlenecks of MWEM are also inherent to the private linear query release problem. Lower bounds \citep{ullman2011pcps} exist showing that no algorithm can answer more than $n^{2+o(1)}$ queries with constant accuracy in time polynomial in \textit{both} the dimension of the data \textit{and} $n$. Our results improve the dependency on the number of queries, but maintain an exponential dependence on the data dimension. Combining our approach with heuristics designed for ameliorating this curse of dimensionality is an interesting future direction.

\section{Preliminaries} % Champ
We give the formal definition of differential privacy below:
\begin{definition}
Let $\mathcal{A}$ be any randomized algorithm that operates on databases whose elements come from some universe. For parameters $\varepsilon > 0$ and $\delta \in [0,1]$, the algorithm $\mathcal{A}$ is $(\varepsilon,\delta)$--\textbf{differentially private (DP)} if for any two neighboring databases $S \sim S'$ (ones that differ on one row only) and for any subset of outcomes $T$ of the output space, we have:
    \begin{align*}
        \Pr[\mathcal{A}(S)\in T] \leq e^\varepsilon\cdot \Pr[\mathcal{A}(S')\in T] + \delta.
    \end{align*}
\end{definition}

\begin{definition}
The \textbf{exponential mechanism (EM)} \citep{mcsherry2007mechanismdesign} solves the \textit{private selection} problem: Given a set of candidate options $R$ and a \textit{score function} $s:\mathcal{D}\times R\to \mathbb{R}$, we seek to select the candidate with close-to-maximal score. Let the \textit{global sensitivity} of $s$ be defined as:
\begin{align}
    \Delta := \max\limits_{h \in R}\max\limits_{D\sim D'}\left|s(D,h)-s(D',h)\right|.
\end{align} Given a privacy parameter $\varepsilon$, EM samples a candidate according to a categorical distribution:
\begin{align}
    \Pr[\text{$h$ is selected}\mid D] \propto \exp\left(\frac{\varepsilon \cdot s(D, h)}{2\Delta}\right).
\end{align}
 \end{definition}
EM has the following formal guarantees:
\begin{theorem}
The Exponential Mechanism is $\varepsilon$-DP and runs in  $O(|R|)$ time. Utility-wise, if EM outputs some candidate $\widehat{h} \in R$, then:
\begin{align}
\mathop{\Pr}_{\widehat{h}}\left[s(D,\widehat{h}) < q_{\max} - \frac{2\Delta(\ln |R| + t)}{\varepsilon}\right] \leq e^{-t}
\end{align}
where $q_{\max} = \max\limits_{i \in R} s(D,h_i)$ is the maximum score of any candidate. 
\end{theorem}

\section{Method \& Framework} % Champ
In this section we present our framework for accelerating MWEM. For ease of presentation, we first focus on the problem of answering linear queries privately.
\subsection{Setting} % Steve
Let $\mathcal{X}$ be a domain set and $X = \{x_1,...,x_n\}\subseteq \mathcal{X}^n$ be a dataset. The goal is to answer a collection of \textit{linear queries}, which are functions $\phi:\mathcal{X}\to[0,1]$. Let us define 
$$
\phi(X) := \frac{1}{n}\sum\limits_{i=1}^n \phi(x_i).
$$
The \textit{histogram representation} $h \in [0,1]^{|\mathcal{X}|}$ of dataset $X$ is defined as:
$$
h_x = \frac{|\{i\in[n]:x_i=x\}|}{n},\text{ for all } x \in \mathcal{X}.
$$
We can view a linear query as a vector in $[0,1]^{|\mathcal{X}|}$. Then we have that:
$$
\phi(X) = \langle \phi, h\rangle.
$$
Then, if we are given $m$ queries $q_1,...,q_m \in [0,1]^{|\mathcal{X}|}$, the answers to all of them can be captured by the product $Q\cdot h$, where $Q=\{q_1,\dots,q_m\} \in[0,1]^{m\times |\cX|}$.

The goal of \textit{Private Linear Query Release} is to design a randomized algorithm $\mathcal{M}$ that satisfies $(\epsilon, \delta)$-differential privacy. Given privacy parameters $\epsilon, \delta > 0$ and an error bound $\alpha > 0$, the algorithm should output a synthetic distribution vector $\hat{p} \in [0,1]^{|\mathcal{X}|}$ such that, with high probability, the maximum error over all queries is bounded by $\alpha$:
$$
\| Q\hat{p} - Qh \|_\infty = \max_{i \in [m]} | \langle q_i, \hat{p} \rangle - \langle q_i, h \rangle | \leq \alpha.
$$

% (Mention how the exponential mechanism selects)
% for some score function $q$:
% $$
% \mathrm{score}(i) = |q_i\cdot \hat p - q_i
% \cdot h|
% $$

% (Mention how the multiplicative weight update rule works)
% for some learning rate $\eta$:
% $$
% p^{(t)}_i \propto p^{(t-1)}_i\cdot (1-\eta)^{c^{t}_i}
% $$
\subsection{The MWEM algorithm}
As a warm-up, let us review the classic MWEM algorithm for solving the private linear query release problem. Intuitively, the iterative MWU part acts as an ``analyst'' and continuously refines a synthetic data distribution $\widehat{p}$ that achieves low error. The exponential mechanism acts as private ``adversary'', finding queries with high error against the chosen synthetic data distribution.

\begin{algorithm}[h]
\begin{algorithmic}
\caption{MWEM Algorithm for Linear Queries}
\label{alg:mw-em}
    \State \textbf{Inputs: }Queries $Q=\{q_i\} \in [0,1]^{m\times |\mathcal{X}|}$, Dataset $X\in \cX^n$ where $|X| = n$ with histogram $h$. Parameters $\alpha, \varepsilon,\delta >0$.
    \State Initialize $w^{(1)} \gets \vec{1}^{|\mathcal{X}|}$
    \State $T= 4\alpha^{-2}\ln m, \varepsilon_0 = \varepsilon(T \ln\frac{1}{\delta})^{-1/2}$ and $ \eta=\sqrt{\frac{\ln |\cX|}{T}}$ \State Let $u(h, \hat{p}, q_i):=|\langle q_i, h\rangle-\langle q_i, \widehat p\rangle|$.
    \Procedure{MWU}{}
    \For{$t=1$ to $T$}
        \State Let $i_t \gets $ \textsc{EM}$(\varepsilon_0,\frac{w^{(t)}}{||w^{(t)}||},u, \frac{1}{n})$ and $c^{(t)} \gets q_{i_t}$.
        \State For each $i \in \mathcal{X}$: $w_{i}^{(t+1)}=w_{i}^{t}\cdot e^{-\eta \cdot c_{i}^{(t)}}$.
    \EndFor
    \State \textbf{Output} $\widehat{p}:=\frac{1}{T}\sum_{i=1}^T p^{(t)}$
    \EndProcedure
    \Procedure{EM}{$\varepsilon_0,\hat{p},h,u, \Delta_u$}
       \State\Return index $i \in [m]$ w.p. $\propto \exp\left(\frac{\eps_0\cdot u(h, \hat{p}, q_i)}{2\Delta_u}\right)$.
    \EndProcedure
\end{algorithmic}
\end{algorithm}

We present MWEM as \Cref{alg:mw-em} in two parts: the MWU \textit{iterative part}, and  the EM ``oracle''. This formulation is convenient for our presentation because our improvements will mainly be focused on the \textsc{EM} procedure.

The guarantees of MWEM are 
established by \citet{hardt2012simple}. For the runtime, each iteration takes $O(m|\cX|)$ time, and there are $T=O(\alpha^{-2}\log m)$ iterations.
\begin{theorem}
\label{thm:mw-em-original-guarantees}
Suppose we are given a dataset $X \in \mathcal{X}^n$ with $n = \Omega(\eps^{-1}\alpha^{-2}\log\frac1\delta\log\frac{m}{\beta} \sqrt{\log|\mathcal{X}|})$. The MWEM Algorithm is $(\varepsilon,\delta)$--DP and outputs a synthetic dataset $\widehat{p} \in \Delta(\mathcal{X})$ such that with probability at least $1-\beta$ it holds that
$$
||Q\cdot (\widehat{p}-h)||_{\infty} \leq \alpha.
$$
The runtime of each iteration is $O(m|\cX|)$.
\end{theorem}

\subsection{``Lazy'' Exponential Mechanism}
To arrive at our improvement, we first recall that the exponential mechanism can be implemented using the classic Gumbel-Max-Trick (see \Cref{sec:gumbel-distrib}).
\begin{lemma}
\label{lemma:gumbel-max-trick}
Let $x_1 \geq x_2 \geq \cdots \geq x_n$ be a list of real numbers, and define the distribution $p \in \Delta([n])$\footnote{We use the $\Delta(\cdot)$ notation to denote the probability simplex over a domain.} where $p_i \propto \exp(x_i)$. Consider sampling $n$ random variables %$\{G_i\} 
$G_1, \dots ,G_n\sim \mathrm{Gumbel}(0, 1)$ and let $$\hat{i} \in \arg \max_{i \in [n]} (x_i + G_i).$$
Then, $\hat{i}$ is distributed according to $p$. 
\end{lemma}

Though this trick helps with numerical stability, sampling every $G_i$ and taking the maximum still takes linear time. \citet{mussmann2017fast} demonstrated that, if we possess knowledge of the $k=\sqrt{n}$ largest numbers $S_k = \{x_1,...,x_k\}$, then sampling only $\Theta(\sqrt{n})$ Gumbel random variables suffices to sample from $p$. We will refer to this method as \textit{lazy Gumbel Sampling} and we include more background information on it in \Cref{sec:lazy-gumbel-sampling-appx}.

Computing $S_k$ in general, however, also requires linear time. To get sublinear runtime, we require that the numbers have some specific structure. For example, when the $x_i$-s are inner products between a continuously evolving query vector $q$ and a static vector dataset $V = \{v_1,...,v_n\}$, we can obtain $S_k$ by solving the \textit{$k$-Maximum Inner Product Search} problem ($k$-MIPS) with dataset $V$ and queries $q$. The $k$-MIPS problem is well-studied, particularly as it can be reduced to the $k$-NN problem \citep{neyshabur2015symmetric}. For more information, see \Cref{sec:mips-to-knn}.

Conveniently, this inner product structure arises very naturally in the MWEM settings we consider. For linear queries specifically, recall that our scores are indeed of the form $\langle q_i,h-\widehat{p}\rangle$, where $h-\widehat{p}$ is provided continuously by MWU and $Q=\{q_i\}_{i=1}^m$ are fixed in the beginning. Thus, we can first pre-process $Q$ to build a $k$-MIPS index $\mathcal{H}$ and then query $\mathcal{H}$ in each iteration to obtain the top $\sqrt{m}$ inner products. 

In this section we will state our algorithm in a manner agnostic to the implementation of index $\mathcal{H}$, though this will be a vital detail in our experimental section. For the theoretical analysis it will be enough to assume that each query to $\mathcal{H}$ to approximately retrieve $S_k$ can be answered in $ O(k)$ time. Indeed, the retrieval process itself can be lossy, a situation common to many fast $k$-MIPS algorithms, and we will specifically address the consequences of only \textit{approximately} obtaining $S_k$ in \Cref{sec:approximation-mips}.

\subsection{Fast MWEM with a perfect index}
Following our previous discussion, we propose Fast-MWEM as \Cref{alg:fast-mwem-queries} for solving the linear query release problem. In that problem specifically, the scores in the exponential mechanism are actually the \textit{absolute values} of inner products. To use a $k$-MIPS data structure we have to augment our query dataset by making it closed to complements: if $q_i \in Q$ then $1-q_i \in Q$. 

To aid presentation, we first analyze the algorithm under the assumption that the $k$-MIPS index $\mathcal{H}$ is \textit{perfect}, meaning that it \textit{exactly} finds the set $S_k$. We denote by $\mathcal{H}(k,h-p)$ the result of querying $\mathcal{H}$ with vector $h-p$ to find the query set $S_k \subset Q$ with the $k$-largest inner products $\langle q_i,h-p\rangle$.
\begin{algorithm}
    \caption{Fast MWEM}
    \label{alg:fast-mwem-queries}
    \begin{algorithmic}[1]
        \State \textbf{Inputs: }Queries $Q=\{q_i\}_{i=1}^{m}$, Dataset $X\in \mathcal{X}^n$ with histogram $h$. Parameters $\alpha,\varepsilon,\delta >0$.
        \State {\color{purple}Initialize a $k$-MIPS index $\mathcal{H}$ on $Q$.}
        \State $T\gets \frac{4\ln m}{\alpha^2}, \varepsilon_0 \gets \frac{\varepsilon}{\sqrt{T\cdot \ln\frac{1}{\delta}}}$, and $ \eta=\sqrt{\frac{\ln |\cX|}{T}}$.
    \Procedure{LazyEM}{$\varepsilon_0,p,h,{\color{purple}\mathcal{H}}, u, \Delta_u$}
        \State $S \gets \mathcal{H}(\sqrt{m},h-p)$.
        \State Sample $G_s \sim \text{Gumbel}(0,1)$ for each $s \in S$.
        \State Let $M= \max_{s\in S}\left\{\frac{\varepsilon_0\cdot |\langle q_s, p-h\rangle|}{2\Delta_u}+G_s\right\}$
        \State Let $L= \min_{s\in S}\left\{\frac{\varepsilon_0\cdot |\langle q_s, p-h\rangle|}{2\Delta_u}\right\}$ and {$B = M-L$}
        \State Let $C\sim\text{Bin}(m-\sqrt{m},1-e^{-e^{-B}})$
        \State Sample $C$ queries from $\mathcal{Q}\setminus S$ into set $T$
        % \vspace{1mm}
        \State For $t\in T$, sample $U_t \sim \mathrm{Uniform}(e^{-e^{-B}}, 1)$.
        \State Let $G_t = -\ln(-\ln(U_t))$.
        \State \textbf{return }${\arg}\max\limits_{i\in S\cup T}\left\{\frac{\varepsilon_0\cdot |\langle q_i, p-h\rangle|}{2\Delta_u}+G_i\right\}$
    \EndProcedure
    \end{algorithmic}
\end{algorithm}

Referring back to \Cref{alg:mw-em}, we create $\mathcal{H}$ in the initialization stage, and substitute the \textit{EM} algorithm with our lazy approach. The iterative MWU part remains identical, except that \textit{LazyEM} is now invoked instead of EM. 

We analyze Fast MWEM via the following theorem:
\begin{theorem}
\label{thm:main-thm-acc-mw-em}
Let $\mathcal{H}$ be a data structure that exactly solves the $k$-MIPS problem for any fixed dataset $Q$ of size $m$ with probability at least $1-\frac{1}{m}$ over $T$ queries. Suppose that each query $\mathcal{H}(k,q)$ runs in $O(k)$ time. 

Let $n = \Omega(\log\frac1\delta\log\frac{m}{\beta} \sqrt{\log|\mathcal{X}|}\cdot \eps^{-1}\alpha^{-2})$. Then Algorithm \autoref{alg:fast-mwem-queries} is a $(\varepsilon,\delta+\frac{1}{m})$--DP algorithm that produces a synthetic dataset histogram $\widehat{p} \in \Delta(\mathcal{X})$ such that with probability at least $1-\beta-\frac{1}{m}$, we have $||Q\cdot(\widehat{p}-h)||_\infty\leq \alpha$. 

The expected per-iteration runtime is $\Theta(|\mathcal{X}|\sqrt{m})$.
\end{theorem}
\begin{proof}
Intuitively, if $\mathcal{H}$ answers correctly every time it is queried, then \text{LazyEM} samples from a distribution that is \textbf{identical} to the one required by the exponential mechanism. As a result, we can apply the advanced composition theorem (\Cref{thm:advanced_composition}) over the $T$ iterations of the main routine and conclude that the algorithm is private. However, there is the chance that $\mathcal{H}$ fails, which we have to analyze separately. 

Let $\mathcal{E}_{\text{good},X}$ be the event that $\mathcal{H}$ answers correctly to identify the  $\sqrt{m}$ nearest neighbors of $p^{(t)}$ for all $t \in T$. Let $X \sim X'$ be two neighboring datasets and let $\mathcal{A}$ be Algorithm \ref{alg:fast-mwem-queries}. Let $\gamma := \Pr\left[\overline{\mathcal{E}_{\text{good},X}}\cup\overline{\mathcal{E}_{\text{good},X'}}\right]$ and let $E$ be an event over the output space of $\mathcal{A}$. Note that $|X\cup X'| = m+1$ and so $\gamma = \frac{1}{m}$. We have % \stevenote{I'm pretty sure this would have to be changed slightly to match the ANN proof that Champ wrote.}
\begin{align*}
    &\Pr[\mathcal{A}(X)\in E] \\
    &= \Pr[\mathcal{A}(X)\in E \cap (\mathcal{E}_{\text{good},X}\cap\mathcal{E}_{\text{good},X'})]\\
    &\quad\quad+\Pr\left[\mathcal{A}(X)\in E\cap  (\overline{\mathcal{E}_{\text{good},X}}\cup\overline{\mathcal{E}_{\text{good},X'}})\right]\\
    &\leq \Pr[\mathcal{A}(X)\in E\cap (\mathcal{E}_{\text{good},X}\cap\mathcal{E}_{\text{good},X'})] + \gamma\\
    &\leq e^{\varepsilon}\cdot\Pr[\mathcal{A}(X')\in E\cap (\mathcal{E}_{\text{good},X}\cap\mathcal{E}_{\text{good},X'})]+\delta+\gamma\\
    &\leq e^{\varepsilon}\cdot\Pr[\mathcal{A}(X')\in E]+\delta+\gamma,\tag{*}
\end{align*}
establishing the privacy guarantees in the theorem statement. The inequality (*) follows because if both $\mathcal{E}_{\text{good},X}$ and $\mathcal{E}_{\text{good},X'}$ occur, both in $\mathcal{A}(X)$ and in $\mathcal{A}(X')$ the distributions from which $i_t$ is sampled are exactly identical to the ones from the exponential mechanism which implies that \textsc{LazyEM} is $(\varepsilon_0,0)$--DP. Thus $\mathcal{A}$ is $(\varepsilon,\delta)$--DP, giving the inequality from the definition of DP.

For the runtime, each iteration of \textsc{MWU} now takes $O(|\mathcal{X}|\sqrt m)$ time in expectation, as claimed.

Finally, we settle the utility guarantees. If $\mathcal{E}_{\text{good},X}$ happens, then the utility guarantees are identical to \Cref{thm:mw-em-original-guarantees}. % the original MW-EM algorithm (\autoref{thm:mw-em-original-guarantees}). 
Thus, if
\begin{align*}
    n = \Omega\left(\frac{\log\frac1\delta\log\frac m\beta\sqrt{\log|\cX|}}{\eps\alpha^2}\right),
    % n \geq \frac{c_{\varepsilon,\delta}\cdot\ln(m/\beta)\cdot \sqrt{\ln(|\mathcal{X}|)}}{\alpha^2},
\end{align*}
then with probability at least $1-\beta-\gamma$, we have that
$$
||Q\cdot(\widehat{p}-h)||_\infty\leq \alpha. \quad \qedhere
$$
\end{proof}

\subsection{Imperfect indices}
\label{sec:approximation-mips}
We now drop our simplifying assumption that the index $\mathcal{H}$ is perfect. A $k$-MIPS algorithm satisfying our $O(k)$ query runtime requirement can only provide an approximation to the correct set $S_k$, which we define as follows:

\begin{definition}
A set of indices $S$ is an approximate top $k$ if $|S| = k$ and for some constant $c$ it holds that:
$$
\max_{i \notin S} x_i - \min_{i \in S} x_i \leq c.
$$
\end{definition}

Let us assume that $\mathcal{H}(\sqrt{m},p)$ only returns a $c$-approximate set $S$ in \Cref{alg:fast-mwem-queries}. Our theoretical analysis can still proceed, albeit in two orthogonal directions:

\paragraph{Preserving runtime} We can show that the $O(|\cX|\sqrt{m})$ runtime guarantee is preserved with a small degradation in the privacy parameter. Specifically, in \Cref{thm:ann-runtime-preserving} we show that Lazy-EM with an imperfect index is $(\varepsilon+2c,0)$-DP.

The full analysis is deferred to \Cref{sec:robustness-to-approximation}. At a high level, we analyze the relationship between two distributions over queries: (i) the original exponential mechanism, and (ii) the one induced by an approximate top-$k$ set. In \Cref{lem:ann-privacy} we provide piecewise bounds that demonstrate the closeness of those distributions as a function of $c$.

\paragraph{Preserving privacy} We can avoid any privacy degradation by increasing the number of samples we take. By lowering the margin $B$ by a factor of $c$ we show in \Cref{sec:robustness-to-approximation} that with per-iteration runtime of $\Theta(e^c \sqrt{m})$, the privacy guarantees are preserved. 

In both cases, we demonstrate that our Fast-MWEM framework yields a sublinear time, differentially private algorithm, even in the case of imperfect indices.

\section{Solving LPs Privately, Faster}
\label{sec:primal}
In this section we apply our fast MWEM framework towards approximately solving Linear Programs (LPs) privately. Following \cite{hsu2014privately}, we focus on solving LPs that have the form:
\begin{align*}
    &\max\limits_{x \in \mathbb{R}^d} c^\top x\\
    \text{s.t.}\quad&Ax\leq b
\end{align*}
where $c \in \R^d$, $A \in \R^{m\times d}$ and $b\in \R^m$. Let $\mathcal{K}_{\text{OPT}} := \{x \in \R^d\mid c^\top x = \text{OPT}\}$. We can solve the original LP by repeatedly solving the following \textit{feasibility} problem:
\begin{align*}
    &\text{find}\, x \in \mathcal{K}_{\text{OPT}}\\
    \text{s.t.}\quad &Ax \leq b
\end{align*}
and binary searching the value of $\text{OPT}$. Therefore, we focus only on feasibility LPs. 

\subsection{Scalar Private LPs}\label{sec:scalar-private}
Consider a private database $D$ that defines our problem parameters $A(D),b(D),c(D)$. Let 
$$
\rho = \sup\limits_{D} ||A(D)||_\infty
$$
be a publicly known upper bound to the \textit{width} of our LP. In this section, we consider \textit{low-sensitivity, scalar private} LPs where for neighboring databases $D$ and $D'$ it is true that $||b(D)-b(D')||_\infty \leq \Delta_\infty$. At the same time, $A(D)=A(D')$ and $c(D)=c(D')$. Other types of neighboring and sensitivity conditions can also be considered \citep{hsu2014privately}. Our method can be applied to those as well.

We further assume for simplicity that that the feasible solution is a distribution:
\begin{align*}
    &\text{find}\, x\in\mathbb{R}^d_+\\
    \text{s.t.}\quad&||x||_1 = 1\,\text{and}\,Ax\leq b
\end{align*}

Solving LPs approximately with the MWU method is a classic result: We start with the uniform distribution which we refine using the MWU rule. Every time we have a candidate solution $\widetilde{x}^{(t)}$, we find the worst constraint $p^{(t)}$, which is the row $A_j$ maximizing $A_j \widetilde{x}^{(t)}-b_j$. Our losses are then the entries of that constraint. 

To achieve privacy, we use the exponential mechanism to select the worst constraint with quality score
$$
Q_t(i,b(D)) = A_i \widetilde{x}^{(t)} -b_i.
$$
By \Cref{thm:advanced_composition}, our entire algorithm is $(\varepsilon,\delta)$-DP. Analyzing the error of EM the solution $x^*$ we obtain has $Ax^* \leq b + \alpha\cdot \vec{1}$ with probability at least $1-\beta$ \citep{hsu2014privately}, where
$$
\alpha := \alpha_{\text{err}}=\widetilde{O}\left(\frac{\rho^{1/2}\Delta_\infty^{1/2}}{\varepsilon^{1/2}}\log^{1/4}d\log^{1/2}\frac{1}{\beta}\right).
$$
Each iteration in this algorithm takes $\Theta(dm)$ time.

\paragraph{The improvement}
Our fast-MWEM framework can improve this runtime via the familiar observation that scores can be written as inner products. Specifically, consider the vector set $\{A_i \circ b_i\}_{i=1}^m \in (\mathbb{R}^{d+1})^m$, where $\circ$ denotes concatenation. Then we can write:
$$
Q_t(i,b(D)) = \langle A_i\circ b_i,\widetilde{x}^{(t)} \circ -1\rangle.
$$
We can then run \textsc{Lazy-EM} with queries $x'_t = \widetilde{x}^{(t)} \circ -1$.
\begin{theorem}
    In the scalar-private, low-sensitivity setting, there exists an $(\varepsilon,\delta)$-private algorithm for solving the feasibility LP $Ax \leq b$ that outputs a solution $x^*$ such that with probability at least $1-\beta$ we have $Ax^*\leq b+\alpha_{\text{err}}\cdot \vec{1}$. The runtime per iteration is $O(d\sqrt{m})$ in expectation.
\end{theorem}

The algorithm is described below as \Cref{alg:mwu-lp}.
\begin{center}
\begin{algorithm}
    \caption{Fast Scalar Private LP Solver}
    \label{alg:mwu-lp}
    \label{alg:accelerated-mw-em}
    \begin{algorithmic}[1]
    \State \textbf{Input:} $A \in \R^{m\times d}$, $b \in \R^m$, error $\alpha\in (0,1)$, and privacy parameters $\eps>0,\delta\in(0,1), \Delta_\infty\in \R$.
    \Procedure{Initialization}{}
        \State Let $\widetilde{x}^{(1)}$ be the uniform distribution in $\R^d$.
        \State Let $\rho := \max_{ij}|A_{ij}|$ be the width of the LP.
        \State Let $\alpha > 0$ be the desired accuracy.
        \State Set $T \gets \frac{9\rho^2 \log d}{\alpha^2}$, $\eta \gets \sqrt{\frac{\log d}{T}}$,
        
        \vspace{-0.5mm}
        
        $\varepsilon_0 \gets \frac{\varepsilon}{\sqrt{8T\log(1/\delta)}}$, and $Q(i, b) = A_i\widetilde x^{(t)}-b_i$.
        \State Construct $k$-MIPS index $\mathcal{H}$ with $\{A_i\circ b_i\}_{i=1}^m$.
    \EndProcedure
    
    \Procedure{MWU}{}
        \For{$t = 1$ to $T$}
            % \State \Comment{Run \textsc{LazyEM}}
            \State \textit{Score: $Q_t(i,b(D)) := \langle A_i\circ b_i,\widetilde{x}^{(t)} \circ -1\rangle$,}
            \State $p^{(t)} = \text{LazyEM}(\varepsilon_0,-x'_t,0,\mathcal{H}, Q, \Delta_{\infty})$
            % \vspace{2mm}
            % \State \Comment{Perform MWU Updates}
            \State $\ell_i^{(t)} \gets \frac{1}{\rho} A_{p^{(t)}i}$ for each $i \in [d]$.
            \State $X^{(t+1)}_i = e^{-\eta \ell_i^{(t)}} X^{(t)}_i$ for $i \in [d]$.
            \State $\widetilde{x}^{(t)} = \frac{X^{(t+1)}}{||X^{(t+1)}||_1}$.
        \EndFor
        \State Output $\bar{x} = \frac{1}{T}\sum_{t=1}^T \widetilde{x}^{(t)}$
    \EndProcedure
    \end{algorithmic}
\end{algorithm}
\end{center}

\subsection{Constraint-Private LPs via Dense MWU} 
\label{sec:dual}
Another way to solve LPs approximately is to focus on the dual LP instead and apply MWU to compete with the best \textit{dense} distribution in hindsight. This gives us a solver that finds an approximate solution $x^*$, though it might not satisfy a small fraction of the constraints. For background on dense distributions and Bregman projections, see \Cref{sec:bregman}.

The Dense MWU algorithm operates on the dual space: instead of proposing a solution $\widetilde{x}^{(t)}$ and computing the worst constraint for $\widetilde{x}^{(t)}$, the algorithm suggests a distribution $\widetilde{y}^{(t)}$ over constraints and computes a solution $x^*$ with minimal expected constraint violation. Computing $x^*$ is done via the dual oracle:
\begin{definition}
    Let $\rho \geq \sup_D \sup_{x \in \mathbb{R}^d} ||A(D)x-b(D)||_\infty$ be a data-independent LP width. Given a distribution $y \in \Delta([m])$ over constraints, the \textbf{$(\alpha,\beta)$- dual oracle} finds some candidate solution $x^* \in \mathbb{R}^d$ such that w.p. at least $1-\beta$,
    \begin{align*}
        \sum\limits_{i=1}^m y_i (A_i x^*) &\leq \min\limits_{x \in \mathcal{K}} \sum\limits_{i=1}^m y_i (A_i x) + \alpha\,\text{ and}\\
        ||Ax^* - b||_\infty &\leq \rho.
    \end{align*}
\end{definition}

After updating $\widetilde{y}^{(t)}$ via MWU, we project it to the space of $\frac{1}{s}$-dense distributions via the Bregman projection $\Gamma_s$ at each step. It is known \citep{herbster2001tracking} that this gives low regret guarantees against the best, in-hindsight, dense distribution: 
\begin{lemma}
    If $||\ell^{(t)}||_\infty \leq 1$ and $\eta \leq \frac{1}{2}$ then it is true that:
    \begin{align*}
        \frac{1}{T}\sum\limits_{t=1}^T\langle \ell^t, \widetilde{B}^t\rangle - \frac{1}{T}\sum\limits_{t=1}^T\langle \ell^t,\widetilde{B}^*\rangle + \eta + \frac{\log n}{\eta T}
    \end{align*}
    when $\widetilde{B}^*$ is any distribution uniform on a subset $S \subseteq [n]$ with $|S| \leq s$
\end{lemma}
In other words, we can get the same guarantee as the classic MWU, but while only using dense distributions!
% \footnote{We can now only compete against dense distributions.}

% Set $\eta = \sqrt{\frac{\log n}{T}}$, $T=\frac{36\rho^2\log n}{\alpha^2}$ and assume that we $(\alpha/3,\beta/T)$-oracle available and that $0 \leq \alpha \leq 9\rho$. 
\cite{hsu2014privately} prove that dense MWU finds $x\in \mathcal{K_{\text{OPT}}}$ such that \textit{all but $s-1$} constraints are violated with error at most $\alpha$.

\paragraph{Constraint Privatization}
In a constraint private LP, we assume that for neighboring datasets $D \sim D'$ the constraints $(A(D),b(D))$ and $(A(D'),b(D'))$ differ in that one of them has one extra row. Note that we no longer have a low-sensitivity assumption. Suppose our oracle is $\varepsilon'=\varepsilon/(2T\log(1/\delta))^{1/2}$ private when on neighboring instances $y,y'$ we have
$$
||y||_\infty, ||y'||_\infty \leq \frac{1}{s},\text{ and } ||y-y'||_1 \leq \frac{2}{s}.
$$
Then, by advanced composition the whole algorithm is $(\varepsilon,\delta)$-private. When an extra row is added, it is known that the Bregman projections between $y$ and $y'$ are close (\Cref{lemma:sensitivity-bregman}), so privacy follows.

\paragraph{Fast-MWEM on the Dual Oracle}
Our task becomes efficiently and privately implementing the dual oracle. Let us assume that the entries of matrix $A$, as well as the constraint $c$ are positive (packing / covering LP). Recall that the oracle $\mathcal{O}$ wants to output, given some $y \in \Delta([m])$,
$$
\mathop{\arg\min}_{x\in\mathcal{K}} \{y^TAx\},
$$
where $\mathcal{K} = \{x\in\mathbb{R}^d_+:c^T x=\text{OPT}\}$. This is a linear program itself, so by the fundamental theorem of linear programming its solution has to be a vertex of the convex polytope $\mathcal{K}$. These vertices have the form $v^{(j)} = \frac{\text{OPT}}{c_j}\cdot e_j$, so our optimization problem becomes
\begin{align*}
    \text{find}\quad \mathop{\arg\min}\limits_{j\in[d]}\left\{\frac{\text{OPT}}{c_j}\cdot y^T A \vec{e_j}\right\}.
\end{align*}
To solve it privately, we use the exponential mechanism with score function
\begin{align*}
    Q(j,y) = -\frac{\text{OPT}}{c_j}y^TA_{:,j}.
\end{align*}
If we let $N_j := -\frac{\text{OPT}}{c_j}\cdot (A^T)_j$ then we just have to maximize $\langle y, N_j\rangle$, over $j \in [d]$.

Note that the vectors $N_j$ are fixed and can be preprocessed in the beginning of the algorithm. Every iteration we query our index $\mathcal{H}$ with $y_j$ to implement \textsc{LazyEM}. For the proof, see \Cref{sec:proof-of-dense-mwu-fast}.
\begin{theorem}
\label{thm:dense-mwu-fast}
If $\rho$ is the width of our LP and $\alpha \leq 9\rho$, then there exists an algorithm that finds some $x^*$ such that $A_i x^* \leq b_i+\alpha$ except for at most $s$ constraints with
$$
s = \widetilde{O}\left(\frac{OPT^2\log\frac{1}{\beta}}{c_{\min}^2\alpha^2\varepsilon}\right).
$$
where $c_{\text{min}}$ is the minimum value of $c$. The algorithm has an expected per-iteration runtime scaling with $O(m\sqrt{d})$, as opposed to the prior $O(md)$.
\end{theorem}

\section{Experiments} % Steve
We validate that our approach gives improvements in the runtime of MWEM, while preserving accuracy and privacy. We fix $\eps=1$ and $\delta=10^{-3}$ throughout.

For our experiments, we consider three implementations of the $k$-MIPS index $\mathcal{H}$. Our baseline is a \textit{flat index}, which simply performs a linear scan over all the vectors. While inefficient, this index can help measure the error of our algorithm when compared to the original MWEM. To improve the runtime, we consider IVF and HNSW indices, the implementations of which we borrow from the FAISS library \citep{douze2024faiss} (see \Cref{sec:index-config}).

\begin{figure}[t]
    \centering
    \includegraphics[scale=0.5]{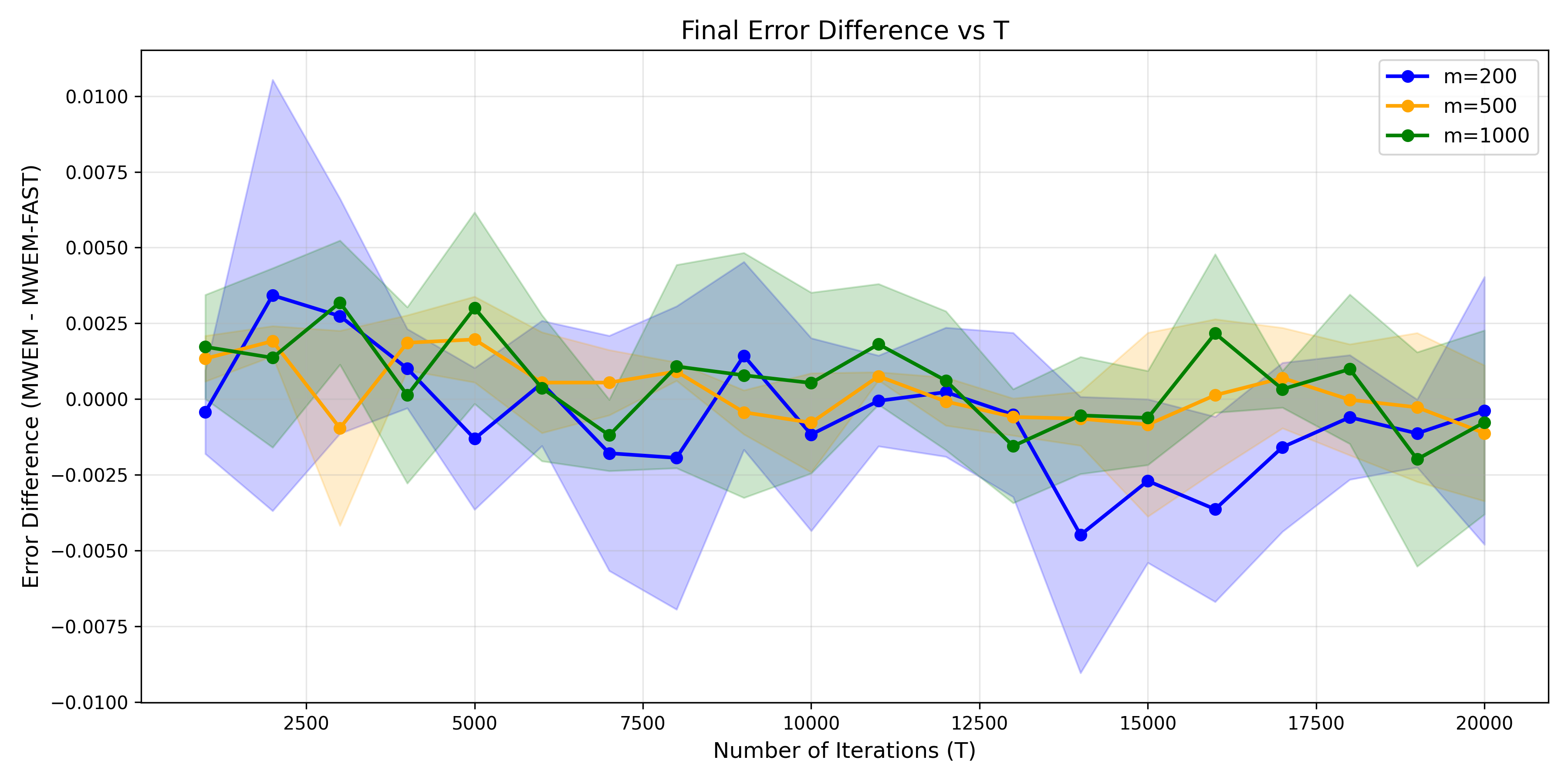}
    \caption{The error difference between MWEM and Fast MWEM}
    \label{fig:error-diff}
\end{figure}

\subsection{Fast Private Linear Query Release}
We first consider the linear queries problem. Fixing the domain size $U=|\cX|=3000$, we generate the true data histogram by sampling $n=500$ points from $\mathcal{N}(\frac{U}{3}, \frac{U}{15})$. Each query is a binary vector of length $|\cX|$ generated by sampling $U/4$ points from $\mathcal{N}(\frac{U}{2}, \frac{U}{5})$.

To test the accuracy of Fast MWEM we first measure, for $m \in \{200,500,1000\}$, the difference in the error between Fast MWEM (flat) and MWEM over $T=20000$ iterations. \textbf{In \Cref{fig:error-diff} we consistently observe that the difference in the two algorithms' error is close to zero}, which aligns with our theoretical expectations.

We also study the error of Fast MWEM throughout multiple iterations when different indices are used. \textbf{We observe in \Cref{fig:error-indices} that all indices achieve about the same error with the flat index, and that error decreases as $T$ increases}, which is predicted by our theory.
\begin{figure}[h]
    \centering
    \includegraphics[width=1\linewidth]{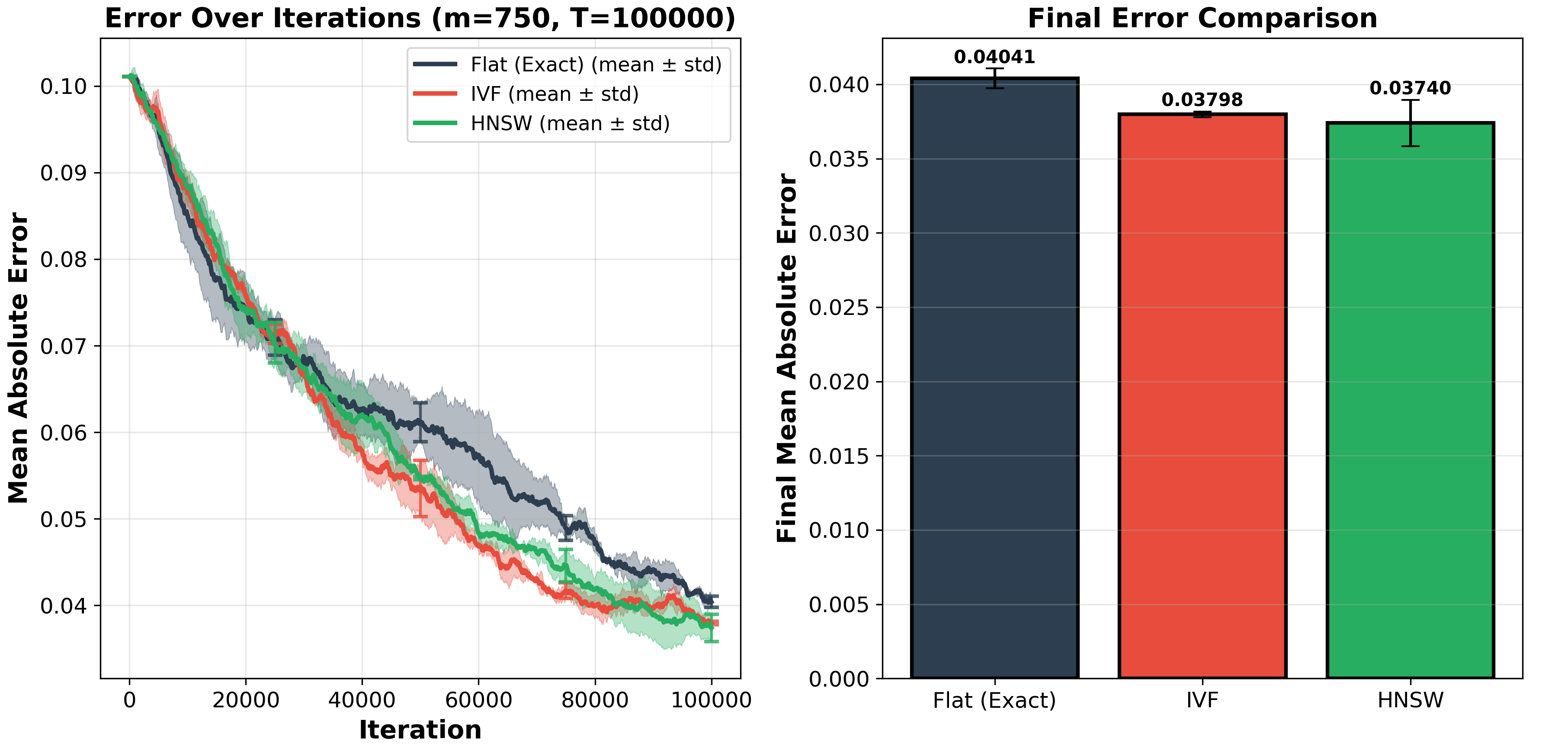}
    \caption{Error over iterations for different indices}
    \label{fig:error-indices}
\end{figure}

To evaluate performance we test $10^4 \leq m \leq 10^5$ for the different indices. As expected, \textbf{the flat index gives almost linear runtime scaling, while the other indices give sublinear scaling}. In particular, the HNSW index notably yields the fastest runtime, which matches the $O(\sqrt{m})$ baseline predicted by our theory (\Cref{fig:index-perf}). 
\begin{figure}[h]
    \centering
    \includegraphics[scale=0.1]{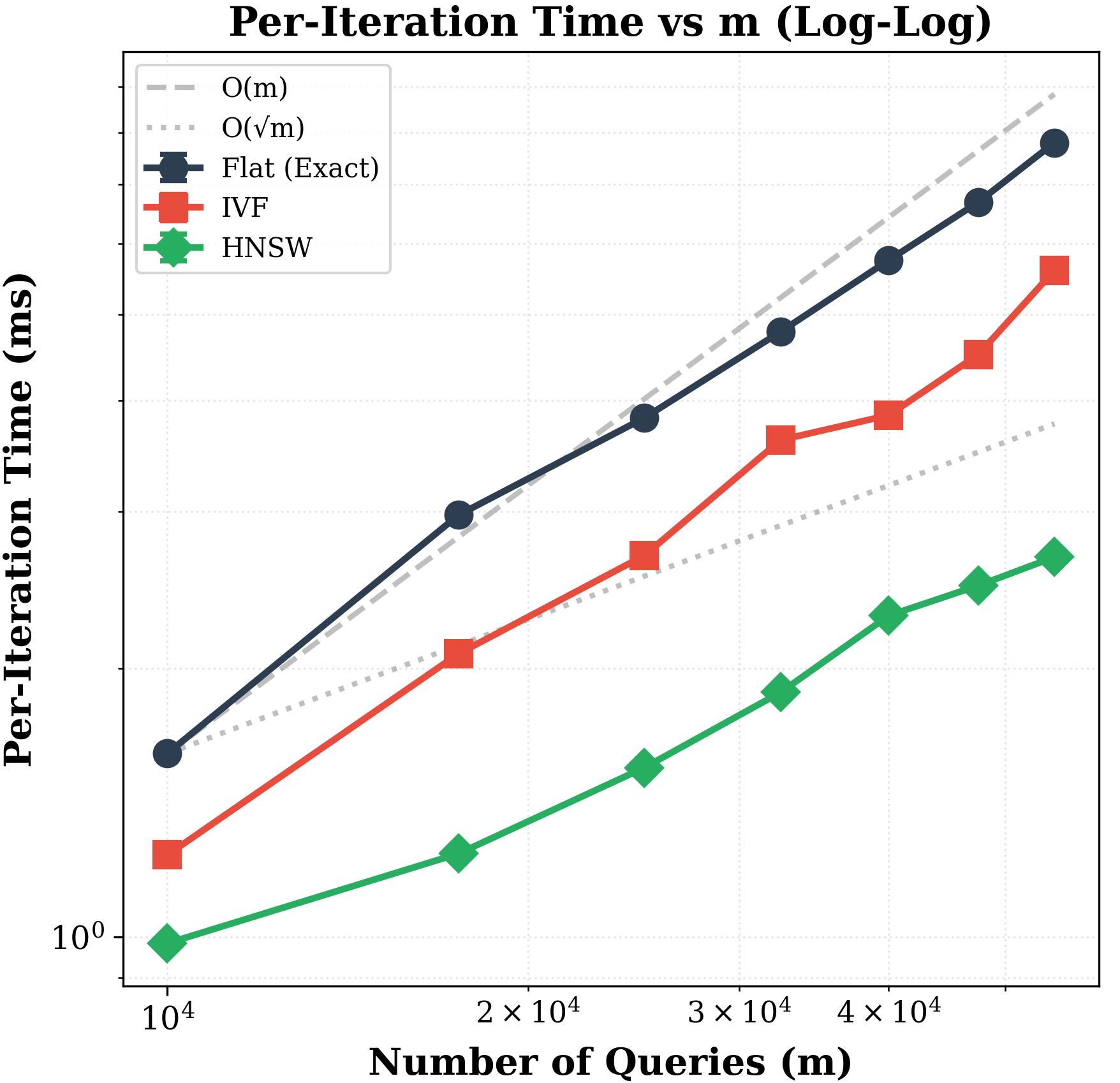}
    \caption{The performance of Fast MWEM for different indices.}
    \label{fig:index-perf}
\end{figure}

Finally, in \Cref{sec:more-experiments-linear-queries} we provide additional experimental results verifying our theory. We study the error as a function of the samples $n$ and also the effect of the margin parameter $B$ on the error of Fast MWEM.

\subsection{Fast Scalar Private LP solving}
We generate a feasibility LP instance by creating a large random matrix $A\sim \mathcal{N}(0,I_{m\times d})$ along with a solution $x^*\in\Delta([d])$. We solve the feasibility LP $Ax\leq \beta$ where $\beta := Ax+\delta$ for some random perturbation vector $\delta$. Throughout all experiments, we fix $d=20$, $\Delta_\infty=0.1$, and $\alpha=0.5$. 

With $T=5000$ we track the fraction of violated constraints compared to the baseline MWEM algorithm. \textbf{Fast MWEM behaves almost identical with MWEM}, which implies that Fast MWEM has a minimal effect on the accuracy of MWEM for solving LPs (\Cref{fig:lp-error}).
\begin{figure}[h]
    \centering
    \includegraphics[scale=0.1]{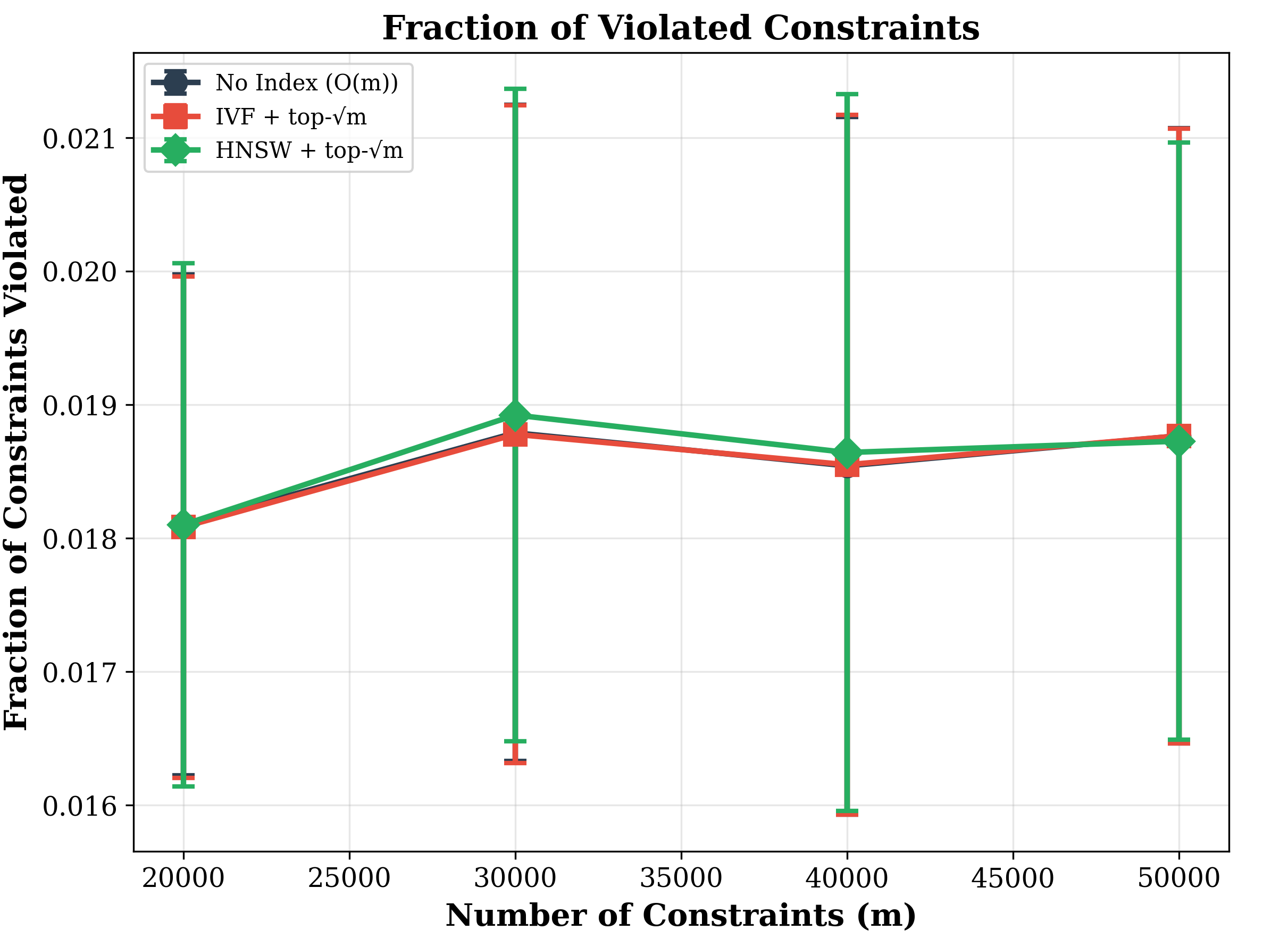}
    \caption{Violated constraints across different indices.}
    \label{fig:lp-error}
\end{figure}

We also perform ablation tests with $3\cdot 10^5 \leq m \leq 15\cdot 10^5$, comparing three indices to each other and to the classic MWEM solution. \textbf{We observe in \Cref{fig:runtime-lp} that MWEM with HNSW vastly outperforms the linear flat index, scaling indeed as $O(\sqrt{m})$ (up to constants).} We do not observe a speed-up when using IVF in this case. We provide additional experiments and plots in \Cref{sec:more-experiments-lps}.
\begin{figure}[h]
    \centering
    \includegraphics[scale=0.1]{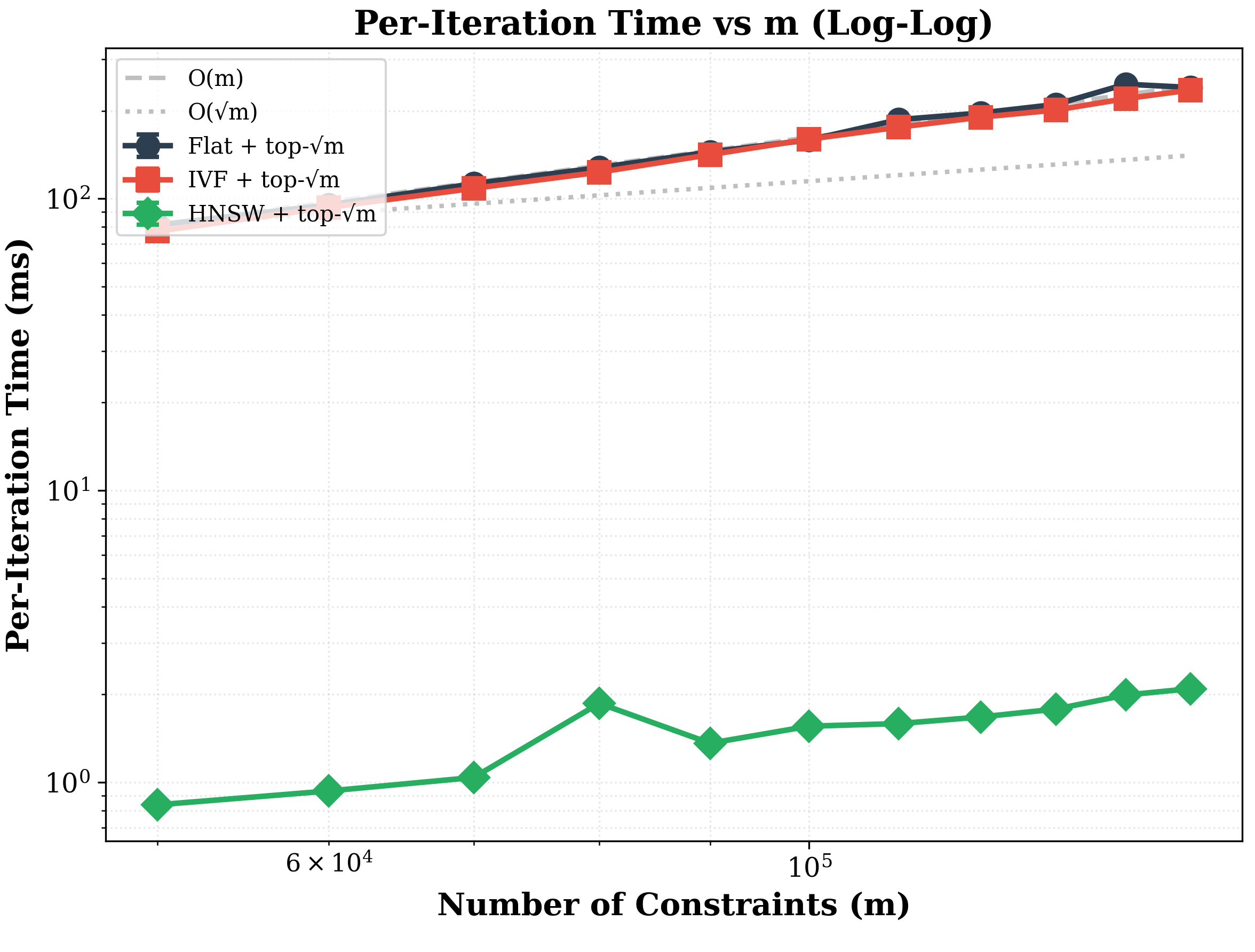}
    \label{fig:runtime-lp}
\end{figure}
\section{Conclusion \& Discussion}
In this work we proposed a new technique for accelerating the celebrated MWEM algorithm, while preserving its runtime and utility guarantees. Empirically, we also showed the benefit of using our method via experiments. Applying our technique in more practical settings to explore its real-world performance is an interesting future direction.

\section*{Acknowledgements}
This work originated as a final project for the course ``Privacy in Statistics and Machine Learning'' (Spring 2025) at Boston University, instructed by Professor Adam Smith. We thank Professor Smith for his guidance, encouragement, and insightful advice throughout the development of this research. We also wish to thank Ta Duy Nguyen for providing helpful comments on the manuscript and for directing our attention to the private linear programming (LP) solver problem.

\bibliography{references}
\bibliographystyle{apalike}

%%%%%%%%%%%%%%%%%%%%%%%%%%%%%%%%%%%%%%%%%%%%%%%%%%%%%%%%%%%%%%%%%%%%%%%%%%%%%%%
%%%%%%%%%%%%%%%%%%%%%%%%%%%%%%%%%%%%%%%%%%%%%%%%%%%%%%%%%%%%%%%%%%%%%%%%%%%%%%%
% APPENDIX
%%%%%%%%%%%%%%%%%%%%%%%%%%%%%%%%%%%%%%%%%%%%%%%%%%%%%%%%%%%%%%%%%%%%%%%%%%%%%%%
%%%%%%%%%%%%%%%%%%%%%%%%%%%%%%%%%%%%%%%%%%%%%%%%%%%%%%%%%%%%%%%%%%%%%%%%%%%%%%%
\newpage
\appendix
\onecolumn
\section{Dense Distributions and Bregman Projections}
\label{sec:bregman}
\begin{definition}[Dense Distributions]
    A distribution $y$ over $[n]$ is $\frac{1}{s}$-\textbf{dense} if $||y||_\infty \leq \frac{1}{s}$.
\end{definition}

Dense MWU requires projecting to the space of dense distributions in the following way:
\begin{definition}[Bregman Projections]
    Let $F:[0,1]^n\to\mathbb{R}$ and $D_F(p,q) := F(p)-F(q)-\langle \nabla F(q), p-q\rangle$ be the $F$-Bregman divergence of $p$ and $q$. We define the \textbf{Bregman projection} of $A \in [0,1]^n$ to the space of $\frac{1}{s}$-dense distributions as the solution to % the following 
    optimization problem
    \begin{align*}
        \Gamma_s A=\mathop{\arg\min}\limits_{{P\text{ is $\frac{1}{s}$-dense}} } D_{\text{KL}}(P\mid\mid A),
    \end{align*}
    where $\text{KL}(x) = \sum_a x_a \log x_a$. Solving this via Lagrange Multipliers we obtain
    \begin{align*}
        \Gamma_sA_a = \frac{1}{s}\min\{1,cA_a\},
    \end{align*}
    where $c$ is such that $s = \sum_a \min \{1,cA_a\}$. 
\end{definition}
\begin{lemma}[Lemma 13 in \citep{hsu2014privately}]
    \label{lemma:sensitivity-bregman}
    Let $A,A'$ be measures on $\mathcal{A}$ and $\mathcal{A}\cup a'$, identical on $\mathcal{A}$. Let $\widetilde{A},\widetilde{A}'$ be the Bregman projections of $A$ and $A'$ into the set of $\frac{1}{s}$-dense distributions. Then,
    $$
    ||\widetilde{A}-\widetilde{A}'||_1 \leq \frac{1}{s}.
    $$
\end{lemma}

\section{Properties of Differential Privacy}
We use two key properties of DP in our work: adaptive composition and post-processing. Adaptive composition \citep{dwork2010boosting} describes a situation when DP algorithms are linked sequentially in an adaptive way.
\begin{theorem}
\label{thm:advanced_composition}
    Suppose algorithms $\mathcal{A}_1,...,\mathcal{A}_k$ are $(\varepsilon, \delta)$--DP. Let $\mathcal{A}'$ be the adaptive composition of these algorithms: on input database $x$, algorithm $\mathcal{A}_i$ is provided with $x$, and, for $i \geq 2$, with the output $y_{i-1}$ of $\mathcal{A}_{i-1}$. Then, for any $\delta' \in (0,1)$, Algorithm $\mathcal{A}$ is $(\widetilde{\varepsilon},\widetilde{\delta})$--DP with
    $$\widetilde{\varepsilon} = \varepsilon\cdot\sqrt{2k\ln(1/\delta')}+2k\varepsilon^2\,\text{ and }\,\widetilde{\delta}=k\delta+\delta'.$$
\end{theorem}
Post-processing guarantees that the applying a data-independent function to the output of a DP algorithm cannot degrade its privacy guarantees:
\begin{theorem}
\label{thm:post processing DP}
Let $\mathcal{A} : U^n \to Y^m$ and $\mathcal{B} : Y^m \to Z^r$ be randomized algorithms, where $U, Y, Z$ are arbitrary sets. If $\mathcal{A}$ is $(\varepsilon,\delta)$--DP, then so is the composed algorithm $\mathcal{B}(\cA(\cdot))$.
\end{theorem}

\section{The Gumbel Distribution and the Gumbel-Max-Trick}
\label{sec:gumbel-distrib}
Suppose we want to sample an index $i \in [n]$ according to a score dataset $\{x_1, x_2, x_3, ..., x_n\}$ with probability $p_i \propto \exp(x_i)$, {where $p_i$ is the probability that we sample $x_i$}. The distribution with $$p_i = \frac{\exp(x_i)}{\sum_{j=1}^{n} \exp(x_j)}$$
is often called a \textit{categorical distribution}, and the Gumbel Max Trick gives us a numerically stable way to sample from it. We begin by defining the Gumbel distribution and reviewing some of its properties.
% First, let us define the Gumbel distribution and review some of its properties.
\begin{definition}[The Gumbel Distribution]
Let $X\sim \mathrm{Gumbel}(\mu,\beta)$. Then, $X$ has expectation of $\mu + \gamma'\beta$, where $\gamma'\approx 0.577$ is the Euler's constant, has variance of $\frac{\pi^2}{6}\beta^2$, and follows the PDF and CDF
$$
f_X(x) = \frac{1}{\beta}\exp(-z-\exp(-z))\text{ and}
$$
$$
F_X(x) = \exp(-\exp(-z)),
$$
respectively, for all $x\in\mathbb{R}$, where $z = \frac{x -\mu}{\beta}$.

\end{definition}

\begin{lemma}
\label{lemma:gumbel-max-trick}
Let $x_1, x_2, x_3, ..., x_n$ be real numbers, and define the distribution $p \in \Delta([n])$, where $$p_i = \frac{\exp(x_i)}{\sum_{j=1}^{n} \exp(x_j)}.$$
Consider sampling $n$ Gumbel random variables %$\{G_i\} 
$G_1, \dots ,G_n\sim \mathrm{Gumbel}(0, 1)$ and let $$\hat{i} \in \arg \max_{i \in [n]} (x_i + G_i).$$
Then, $\hat{i}$ is distributed according to $p$. 
\end{lemma}
\begin{proof}
Let $q_i$ be the probability that $\hat{i} = i$ for all $i \in [n]$. By symmetry, we only need to prove that $q_1 \propto \exp(x_1)$.

{Fix $G_1 = x$. Then, by the Gumbel distribution, $f_{G_1}(x) = e^{-x-e^{-x}}$.}
% Consider a fixed $G_1 = x$, according to the Gumbel distribution, $f_{G_1}(x) = e^{-x-e^{-x}}$. %$\Pr[G_1 = x] = e^{-x-e^{-x}}$. 
For $\hat{i}$ to be $1$, we need $G_i < x_1 + G_1 - x_i$ for $\forall i \in [n]\setminus\{1\}$. Thus, the probability that $G_i < x_1 + G_1 - x_i$ is $e^{-e^{x_i - x_1 - G_1}}$.

Then, by law of total probability,
\begin{align*}
\Pr[\hat{i} = 1] &= \int_{-\infty}^{\infty} f_{G_i}(x) \prod_{i=2}^{n} \Pr[G_i < x] dx\\
&= \int_{-\infty}^{\infty} e^{-x-e^{-x}} \prod_{i=2}^{n} e^{-e^{x_i - x_1 - x}} dx\\
&= \int_{-\infty}^{\infty} e^{x_1-u-e^{x_1 - u}} \prod_{i=2}^{n} e^{-e^{x_i - u}} du &\text{(Substitute $u = x + x_1$)}\\
&= e^{x_1} \int_{-\infty}^{\infty} e^{-u} \prod_{i=1}^{n} e^{-e^{x_i - u}} du\\
&= e^{x_1} C. &\text{(For some constant $C$)}
\end{align*}

Therefore, $q_1 \propto \exp(x_1)$, and $\hat{i}$ is distributed according to $p$.
\end{proof}

Note also the following property of the Gumbel distribution:
\begin{lemma}
\label{lemma:conditional-gumbel}
The probability distribution of $G_1$ where $G_1\sim \mathrm{Gumbel}(0, 1)$ conditioned on $G_1 > B$ is the same as the probability distribution of $G_2$ where $G_2 = -\ln(-\ln(U))$ and $U\sim \mathrm{Uniform}(\exp(-\exp(-B)), 1)$.
\end{lemma}

\begin{proof}
We will prove that the distributions of $G_1$ and $G_2$ are the same by proving that their CDFs are the same namely, $F_{G_1}(x) = F_{G_2}(x)$ for all $x \in \R$.

First, we will determine $F_{G_1}(x)$. When $x \leq B$, $F_{G_1}(x) = 0$. When $x > B$, $F_{G_1}(x) = \frac{\Pr[G_1 \leq x]}{\Pr[G_1 > B]} = \frac{\Pr[G_1 \leq x]}{1 - \Pr[G_1 \leq B]} = \frac{e^{-e^{-x}}}{1 - e^{-e^{-B}}}$. Therefore, $$F_{G_1}(x) = 
\begin{cases}
    0; & x \leq B\\
    \frac{e^{-e^{-x}}}{1 - e^{-e^{-B}}};& B < x\\
\end{cases}
$$
Next, we will determine $F_{G_2}(x)$. Since $U \sim \mathrm{Uniform}(\exp(-\exp(-B)), 1)$, then $$F_{U}(x) = 
\begin{cases}
    0; & x \leq e^{-e^{-B}}\\
    \frac{x - e^{-e^{-B}}}{1 - e ^{-e ^{-B}}};& e^{-e^{-B}} < x \leq 1\\
    1; & 1 < x\\
\end{cases}
$$

After that, we substitute $U$ with $G_2$ and $x$ with $e^{-e^{-x}}$ yielding 
$$F_{G_2}(x) = 
\begin{cases}
    0; & e^{-e^{-x}} \leq e^{-e^{-B}}\\
    \frac{e^{-e^{-x}} - e^{-e^{-B}}}{1 - e ^{-e ^{-B}}};& e^{-e^{-B}} < e^{-e^{-x}} \leq 1\\
    1; & 1 < e^{-e^{-x}}.\\
\end{cases}
$$
Simplifying all expressions result in $F_{G_1}(x) = F_{G_2}(x)$.

\end{proof}

\section{Lazy Gumbel Sampling}
\label{sec:lazy-gumbel-sampling-appx}
In the Gumbel Max Trick of \Cref{lemma:gumbel-max-trick} above, we have to calculate $n$ values of $x_i + G_i$ but since it is very rare that a Gumbel noise will be high, it is very unlikely that we will choose an index $i$ such that $x_i$ is small. Hence, at least intuitively, we often do not need to sample all $n$ Gumbel noise random variables. The following algorithm, due to \cite{mussmann2017fast}, shows how to perform Gumbel Max Trick while only using $\Theta(\sqrt{n})$ samples in expectation.

\begin{algorithm}
\caption{Lazy Gumbel Sampling, \citep{mussmann2017fast}}
\label{alg:lazy-gumbel}
\begin{algorithmic}[1]
\State \textbf{Inputs:} An integer $k \in [n]$, and a score dataset $X = \{x_1, x_2, x_3, ..., x_n\} \in \mathbb{R} ^ n$
\vspace{1mm}
\State Let $S$ be the set of the indices of $k$ largest $x_i$ from $X$
\For{$j \in S$}
    \State Sample $G_j \sim \mathrm{Gumbel}(0, 1)$
\EndFor
\State Let $M = \max_{j \in S}(x_j + G_j)$, $m = \min_{j \in S} x_j$, and $B = M - m$
\State Let $C \sim \mathrm{Bin}(n - k, 1 - \exp(-\exp(-B))$ %be the number of indices in $[n]\setminus S$ that receives a Gumbel noise at least $B$.
\State {Let $T\subseteq [n]\setminus S$ be the set of $C$ distinct indices sampled uniformly from $[n]\setminus S$.}% Sample $C$ indices from $[n]\setminus S$ as a set $T$.
\State Sample $U_i\sim \mathrm{Uniform}(\exp(-\exp(-B)), 1)$, and let $G_i=-\ln(-\ln(U_i))$ for $\forall i\in T$. 

\Comment{{\color{blue}This simulates $G_i\sim \mathrm{Gumbel}(0,1)$ conditioned on $G_i>B$}}
% \State Sample $G_i \sim \mathrm{Gumbel}(0, 1)$, conditioned on $G_i >B$ for $\forall i \in T$. 
\vspace{1mm}
\State \textbf{return} $\arg\max_{i \in S \cup T}(x_i + G_i)$
\end{algorithmic}
\end{algorithm}
This is formalized in the following theorem of \citep{mussmann2017fast}:
\begin{theorem}[\citep{mussmann2017fast}]
\label{thm:lazy-gumbel}
The Lazy Gumbel Algorithm on a score dataset $\{x_1,...,x_n\}$ outputs an index $i \in [n]$ with probability:
$$
p_i \propto \exp(x_i)
$$
using $\Theta(\sqrt{n})$ samples in expectation when $k \approx \Theta(\sqrt{n})$. 
\end{theorem}

\section{From $k$-MIPS to $k$-Nearest Neighbor Search}
\label{sec:mips-to-knn}
The regime we are considering is when the scores $x_i$ are given by inner products of a query vector $q \in \mathbb{R}^d$ with a set of key vectors $\{k_1,...,k_n\} \subseteq \mathbb{R}^d$. In other words, $x_i = \langle q, k_i\rangle$. In addition, we are considering a situation in which we have to sample from $m$ categorical distributions, defined by query vectors $\{q_1,...,q_m\}$ on a fixed set of key vectors. In this setting, identifying the top $k$ inner products can be done efficiently, in total time $O(n+m\sqrt{n})$ rather than $O(nm)$, which is where our improvement lies. 

More specifically, we recognize that this is exactly the $k$--MIPS (Maximum Inner Product Search) problem. This problem has been studied extensively, and we can use any data structure developed for it to solve it. By initializing a data structure like this on the set $\{k_1,...,k_n\}$ and then querying it for each incoming query $q_i$, we get the practical improvement we seek. 

% Let the histogram that our DP interface receive is $p$ and the true histogram is $h_x$, then the score function for a query $\varphi_i$ will be $\langle\varphi_i,p-h_x \rangle$.

% In the lazy Gumbel sampling, we need to find the $k$ queries with largest scores. This is a $k$-Maximum Inner Product Search Problem (MIPS). In this section, we are going to show how to reduce this problem to a $k$-Nearest Neighbour ($k$NN) problem.

We can also make use of the observation that this problem is reducible to the $k$--nearest neighbor problem, which is also very well studied and for which there are also a wide variety of developed algorithms. We elaborate on this briefly below:

We know that 
$$
q^T k_i = \frac{1}{2}(||q||_2^2 + ||k_i||_2^2 - ||q - k_i||_2^2)
$$
and $||q||_2^2$ is a constant for all $i \in [n]$. Therefore, if $||k_i||_2^2$ is also a constant for all $i \in [n]$, then $q^T k_i$ will be largest when $||q - k_i||_2^2$ is smallest. Thus, we will be able to reduce MIPS to $k$NN.

Let $M$ be an arbitrary number that is known to be at least $\max_{i\in[n]} ||k_i||_2^2$. Then, we can transform $k_i$ to 
$$\left[k_i^T, \sqrt{M - ||k_i||_2^2}\right]^T$$so that $||k_i||_2^2 = M$ for all $i \in [n]$. Next, we transform $q$ to $[q^T, 0]^T$ so that it has the same dimension as $k_i$. Since this transformation preserves the dot product $\langle q, k_i \rangle$, we have successfully reduced MIPS to $k$NN.

% \paragraph{$k$-MIPS via concentric LSH}
% Another theoretically sound way to solve $k$-MIPS without reducing it to $k$-NN is by using a concentric LSH construction, as proposed by \cite{mussmann2017fast}. Since lazy Gumbel sampling works even when retrieving the approximate top $k$ inner products, we can use a modification of the Locality Sensitive Hashing data structure:  we maintain a series of copies of the data structure, all tuned to represent concentric regions in $\mathbb{R}^d$. Then, upon arrival of a query, we look at the bucket it is hashed and find the first LSH ``annulus'' containing at least $k$ points co-hashed in the bucket. From that data structure we start outputting the $k$ approximately largest inner products.

\section{Robustness to Approximation}
\label{sec:robustness-to-approximation}
In this section, we show that if our $k$-MIPS algorithm only retrieves the \textit{approximate top-$k$} inner products, we can still perform our improved algorithm.

\begin{definition}
\label{def:approx-top-k}
A set of indices $S$ is an approximate top $k$ if $|S| = k$ and 
$$
\max_{i \notin S} x_i - \min_{i \in S} x_i \leq c,
$$ for some constant $c$.
\end{definition}

\subsection{Preserving Runtime}
Here, we describe how to preserve the $\Theta(\sqrt n)$ runtime under approximation, with some diminished privacy guarantees.
\label{sec:preserving-runtime}
\begin{algorithm}
\caption{Lazy Gumbel Sampling using Approximate Top-$k$ while preserving runtime}
\label{alg:lazy-gumbel-top-k-runtime}
\begin{algorithmic}[1]
\State \textbf{Inputs:} An integer $k \in [n]$, and a score dataset $X = \{x_1, x_2, x_3, ..., x_n\} \in \mathbb{R} ^ n$
\vspace{1mm}
\State Let $S$ be an \textbf{approximate top-$k$} of $X$.
\For{$j \in S$}
    \State Sample $G_j \sim \mathrm{Gumbel}(0, 1)$
\EndFor
\State Let $M = \max_{j \in S}(x_j + G_j)$, $m = \min_{j \in S} x_j$, and $B = M - m$
\State Let $C \sim \mathrm{Bin}(n - k, 1 - \exp(-\exp(-B)))$ %be the number of indices in $[n]\setminus S$ that receives a Gumbel noise at least $B$.
\State {Let $T\subseteq [n]\setminus S$ be the set of $C$ distinct indices sampled from $[n]\setminus S$.}% Sample $C$ indices from $[n]\setminus S$ as a set $T$.
\State Sample $U_i\sim \mathrm{Uniform}(\exp(-\exp(-B)), 1)$, and let $G_i=-\ln(-\ln(U_i))$ for every $i\in T$. 

% \State Sample $G_i \sim \mathrm{Gumbel}(0, 1)$, conditioned on $G_i > B$ for $\forall i \in T$. 
\vspace{1mm}
\State \textbf{return} $\arg\max_{i \in S \cup T}(x_i + G_i)$
\end{algorithmic}
\end{algorithm}
\begin{theorem}\label{thm:ann-runtime-preserving}
    Given the set $S$ of the approximate top-$k$ scores of dataset $X$ in the exponential mechanism with candidate set $[n]$, \Cref{alg:lazy-gumbel-top-k-runtime} runs in expected time $\Theta(\sqrt n)$ and is $(\eps+2c)$-DP.
\end{theorem}
We split the proof of \Cref{thm:ann-runtime-preserving} into two lemmas: \Cref{lem:ann-runtime} proves the runtime guarantees and \Cref{lem:ann-privacy} proves the privacy guarantees.

We first prove the runtime guarantees.
% First, we prove that this algorithm samples at most $\Theta({\sqrt{n}})$ elements in expectation.
\begin{lemma}\label{lem:ann-runtime}
    \Cref{alg:lazy-gumbel-top-k-runtime} runs in expected time $\Theta(\sqrt n)$.
\end{lemma}
\begin{proof}
The proof is the same as the proof of \Cref{thm:lazy-gumbel} since the upper bound on the expected number of samples is the same regardless of the choice of $S$.
\end{proof}

Next, to prove that \Cref{alg:lazy-gumbel-top-k-runtime} is $(\eps + 2c)$-DP, we first prove that the probability of the index $i\in[n]$ returned by \Cref{alg:lazy-gumbel-top-k-runtime} is bounded above and below a constant factor of the true top-$k$ probability. Let $\cI$ and $\cI'$ be \Cref{alg:lazy-gumbel} and \Cref{alg:lazy-gumbel-top-k-runtime}, respectively, and $\cI(X)$ and $\cI'(X,S)$ be the outputs of \Cref{alg:lazy-gumbel} on dataset $X$ and \Cref{alg:lazy-gumbel-top-k-runtime} on dataset $X$ with approximate top-$k$ set $S$, respectively.

\begin{lemma}\label{lem:ann-privacy}
For any score dataset $X = \{x_1, x_2, \dots, x_n\}\in \R^n$, and every $i \in [n]$, we have
$$e^{-c}\Pr[\cI(X) = i] \leq \Pr[\cI'(X,S) = i] \leq e^c\Pr[\cI(X) = i].$$
\end{lemma}

Let $S^*$ be the set of true top-$k$ indices. Define sets $P=S$, $Q = S^*\setminus S$, and $R = [n]\setminus(S\cup S^*)$. Furthermore, define 
$$
A_{P} = \sum_{i\in P}e^{x_i}, \quad\quad A_{Q} = \sum_{i\in Q}e^{x_i},\quad\quad A_{R} = \sum_{i\in R}e^{x_i}.
$$

We split the proof of \Cref{lem:ann-privacy} into 5 claims: \Cref{claim:lem-d2-part-1} and \Cref{claim:lem-d2-part-4} prove the lower bound of $\Pr[\cI'(X,S)=i]$, and \Cref{claim:lem-d2-part-2}, \Cref{claim:lem-d2-part-3}, and \Cref{claim:lem-d2-part-5} prove the upper bound of $\Pr[\cI'(X,S)=i]$.

First, we prove a lower bound on $\Pr[\cI'(X,S)=i]$, if $i\notin Q$.
\begin{claim}\label{claim:lem-d2-part-1}
    For an index $i\in[n]$, if $i\notin Q$, then $\Pr[\cI(X) = i] \leq \Pr[\cI'(X,S) = i]$.
\end{claim}
\begin{proof}
Consider the event $\cI'(X,S)= i$. If $i=\arg\max_{j\in[n]}\{x_j+G_j\}$, then $i$ will be returned by $\mathcal{I}$ because of the Gumbel Max Trick and by $\mathcal{I}'$ because either $i \in S$ or $G_i > B$. If $i\neq\arg\max_{j\in[n]}\{x_j+G_j\}$, there exists some index $j\in Q$ with greater score value than $i$, but with noise value less than $B$, which means that $i$ can still be sampled by $\cI'$, though not by $\cI$. Thus, if $i\notin Q$, we have $\Pr[\cI'(X,S)=i]\geq \Pr[\cI(X)=i]$.
\end{proof}

The next two claims prove an upper bound on $\Pr[\cI'(X, S) = i]$, if $i\in Q$.
\begin{claim}\label{claim:lem-d2-part-2}
    For an index $i\in[n]$, if $Q=\{i\}$, then we have $\Pr[\cI'(X,S)=i]\leq \Pr[\cI(X)=i]$.
\end{claim}
\begin{proof}
Suppose for contradiction that $\Pr[\cI'(X, S) = i] > \Pr[\cI(X) = i]$.

By \Cref{claim:lem-d2-part-1}, for all $j \in [n]\setminus \{i\}$, we have $\Pr[\cI'(X, S) = j] \geq \Pr[\cI(X) = j]$. Then we have
\begin{align*}
    \sum_{j\in[n]}\Pr[\cI'(X,S) = j] &= \Pr[\cI'(X,S)=i] + \sum_{j\in[n]\setminus\{i\}}\Pr[\cI'(X,S)=j]\\
    &\geq \Pr[\cI'(X,S)=i] + \sum_{j\in[n]\setminus\{i\}}\Pr[\cI(X)=j]\\
    &> \sum_{j\in[n]}\Pr[\cI(X)=j] = 1,
\end{align*}
a contradiction.
\end{proof}
\begin{claim}\label{claim:lem-d2-part-3}
    For an index $i\in[n]$, if $i\in Q$ and $|Q|>1$, then we have $\Pr[\cI'(X, S) = i] \leq e^c \Pr[\cI(X) = i]$.
\end{claim}
\begin{proof}
Consider a modified dataset $\widetilde{X} = \{\widetilde{x}_1,...,\widetilde{x}_n\}\in\R^n$, where for all $j\in[n]$, we have
$$\widetilde{x}_j = \begin{cases}
x_j - c &\text{if }j \in Q \setminus \{i\}\\
x_j &\text{otherwise}
\end{cases},$$
and let $\widetilde Q$ be the analogous set to $Q$, but for dataset $\widetilde X$.

We first prove that $\Pr[\cI'(X, S) = i] \leq \Pr[\cI'(\widetilde{X}, S) = i]$. To prove this statement, we fix the Gumbel noise $G_i$ on index $i$. Then, we have 
% and prove that $\Pr[\cI'(X, S) = i | G_i] \leq \Pr[\cI'(\widetilde{X}, S) = i | G_i]$. 
\begin{align*}
    \Pr[\cI'(X, S) = i \mid G_i] &=\prod_{j \in Q \setminus \{i\}} \Pr[\underbrace{G_j < \max\{x_i + G_i - x_j, B\}}_{(1)}\mid G_i] \cdot \prod_{j \notin Q}{\Pr[G_j < x_i + G_i - x_j\mid G_i]}\\
    \tag{for every $j\in Q\setminus\{i\}$, we have $\cI'(X,S)\neq j$ if $x_j+G_j<x_i+G_i$ or $G_j<B$}\\
    &\leq \prod_{j \in Q \setminus \{i\}} \Pr[\underbrace{G_j < \max\{x_i + G_i - (x_j - c), B\}}_{(2)}\mid G_i]\cdot \prod_{j \notin Q}{\Pr[G_j < x_i + G_i - x_j\mid G_i]}\\
    \tag{event $(1)$ is a subset of event $(2)$}\\
    &\leq \prod_{j \in Q \setminus \{i\}} \Pr[G_j < \max\{\widetilde{x}_i + G_i - \widetilde{x}_j, B\}\mid G_i]\cdot \prod_{j \notin Q}{\Pr[G_j < \widetilde{x}_i + G_i - \widetilde{x}_j\mid G_i]}\\
    &= \Pr[\cI'(\widetilde{X}, S) = i \mid G_i].
\end{align*}

Then, by the law of total probability, we have
\begin{align*}
    \Pr[\cI'(X,S)=i] &= \int_{x\in\R}\Pr[\cI'(X,S)=i\mid G_i=x]\cdot f_{G_i}(x)dx\\
    &\leq \int_{x\in\R}\Pr[\cI'(\widetilde{X}, S) = i \mid G_i=x]\cdot f_{G_i}(x)dx\\
    &= \Pr[\cI'(\widetilde{X}, S) = i].
\end{align*}

Since we decrease all $x_{j}$ where $j \in Q \setminus \{i\}$ by $c$, these elements will no longer be in the true top-$k$. Therefore, $|\widetilde Q| = 1$.

Since $|\widetilde Q| = 1$, we have
\begin{align*}
    \Pr[\cI'(\widetilde X, S) = i] &\leq \Pr[\cI(\widetilde X) = i] & \tag{by \Cref{claim:lem-d2-part-2}}\\
    &= \frac{e^{x_i}}{A_P + e^{-c}(A_Q - e^{x_i}) + e^{x_i} + A_R}\\
    &= \frac{e^ce^{x_i}}{e^cA_P + A_Q + (e^c - 1) e^{x_i} + e^cA_R}\\
    &\leq \frac{e^ce^{x_i}}{A_P + A_Q + A_R}\\
    &= e^c \Pr[\cI(X) = i],
\end{align*}
which proves \Cref{claim:lem-d2-part-3}.
\end{proof}
Next, we prove a lower bound on $\Pr[\cI'(X,S)=i]$, if $i\in Q$.
\begin{claim}\label{claim:lem-d2-part-4}
    For an index $i\in [n]$, if $i\in Q$, then we have $e^{-c}\Pr[\cI(X)=i]\leq \Pr[\cI'(X,S)=i]$.
\end{claim}
\begin{proof}
Again, consider a modified dataset $\widetilde X = \{\widetilde x_1,...,\widetilde x_n\}$, where for all $j\in [n]$, we have
$$
\widetilde{x}_j = \begin{cases}
x_j + c &\text{if }j \in S\\
x_j &\text{otherwise}
\end{cases},
$$
and let $\widetilde S^*$ be the analogous set to $S^*$, but for dataset $\widetilde X$. Then, we have $S=\widetilde S^*$.

First we prove that $\Pr[\cI'(\widetilde X, S) = i]\leq \Pr[\cI(X, S) = i]$. Consider fixing $G_1,...,G_n$, such that $\cI'(\widetilde X, S) = i$ is true. Then we know that $\widetilde x_i+G_i = \max_{j\in[n]}\{\widetilde{x}_j+G_j\}$. Then, this implies that $x_i+G_i=\max_{j\in[n]}\{x_j+G_j\}$, since 
\begin{align*}
    \widetilde x_i+G_i &= \max_{j\in[n]}\{\widetilde{x}_j+G_j\}\\
    &= \max\{\max_{j\in S}\{(x_j+c)+G_j\}, \max_{j\notin S}\{x_j+G_j\}\}\\
    &\geq \max\{\max_{j\in S}\{x_j+G_j\}, \max_{j\notin S}\{x_j+G_j\}\}\\
    &= \max_{j\in[n]}\{x_j+G_j\}.
\end{align*}

Thus, the event $\widetilde x_i+G_i = \max_{j\in[n]}\{\widetilde{x}_j+G_j\}$ is a subset of $x_i+G_i=\max_{j\in[n]}\{x_j+G_j\}$, and
$$
\Pr[\cI'(\widetilde X,S) = i]\leq \Pr[\cI(X,S)=i].
$$

Next, we prove $e^{-c}\Pr[\cI(X)=i]\leq \Pr[\cI'(\widetilde X,S)=i]$.

Since $S=\widetilde S^*$, for all $i\in[n]$, we have $\Pr[\cI'(\widetilde X,S)=i]=\Pr[\cI(\widetilde X)=i]$. We then have
\begin{align*}
    \Pr[\cI(\widetilde X)=i] &= \frac{e^{x_i}}{e^cA_P + A_Q + A_R} \tag{Since we added $c$ to every element in $S$}\\
    &\geq e^{-c}\frac{e^{x_i}}{A_P + A_Q + A_R}\\
    &= e^{-c}\Pr[\cI(X)=i].
\end{align*}

Thus, we have 
$$
e^{-c}\Pr[\cI(X)=i]\leq\Pr[\cI'(\widetilde X,S) = i]\leq \Pr[\cI(X,S)=i],
$$
which proves \Cref{claim:lem-d2-part-4}
\end{proof}
Finally, we prove an upper bound on $\Pr[\cI(X,S)=i]$, if $i\notin Q$.
\begin{claim}\label{claim:lem-d2-part-5}
    For index $i\in[n]$, if $i\notin Q$, then $\Pr[\cI(X,S)=i]\leq e^c\Pr[\cI(X)=i]$.
\end{claim}
\begin{proof}
Let $p=\sum_{j\in Q}\Pr[\cI'(X,S)=j]$. Then, by \Cref{claim:lem-d2-part-4}, we have
$$
p\geq e^{-c}\frac{A_Q}{A_P+A_Q+A_R}.
$$

Consider the probability $\Pr[\cI'(X,S)=i]$. We have
\begin{align*}
\Pr[\cI'(X,S)=i] &= \Pr[\cI'(X,S)=i\mid \cI'(X,S)\notin Q]\cdot\Pr[\cI'(X,S)\notin Q]\\
&= \frac{e^{x_i}}{A_P+A_R}\cdot (1-p).
\end{align*}

Consider any element $a\in P$ and $b\in Q$. We have
$$
x_{b}-x_{a}\leq \max_{j\in Q}x_j - \min_{j\in P}x_j \leq c \iff e^{x_b}\leq e^{c+x_a},
$$
which implies 
$$A_Q = \sum_{j\in Q}e^{x_j}\leq \sum_{j\in P}e^{c+x_j}  = e^cA_P,$$
where the inequality holds because $|P|=k\geq |Q|$.

Since $p \geq \frac{e^{-c}A_Q}{A_P + A_Q + A_R}$, we have $1 - p \leq \frac{A_P + (1-e^{-c})A_Q + A_R}{A_P + A_Q + A_R}$. Then, we have
\begin{align*}
    \Pr[\cI'(X, S) = i] &= (1 - p) \cdot \frac{e^{x_i}}{A_P + A_R}\\
    &\leq \frac{A_P + (1-e^{-c})A_Q + A_R}{A_P + A_Q + A_R} \cdot \frac{e^{x_i}}{A_P + A_R}\\
    &= \frac{A_P + (1-e^{-c})A_Q + A_R}{A_P + A_R} \cdot \frac{e^{x_i}}{A_P + A_Q + A_R}\\
    &\leq \frac{A_P + (1-e^{-c})e^cA_P + A_R}{A_P + A_R} \cdot \frac{e^{x_i}}{A_P + A_Q + A_R}\\
    &= \frac{e^cA_P + A_R}{A_P + A_R} \cdot \frac{e^{x_i}}{A_P + A_Q + A_R}\\
    &\leq e^c \cdot \frac{e^{x_i}}{A_P + A_Q + A_R}\\
    &= e^c\Pr[\cI(X) = i],
\end{align*}
which proves \Cref{claim:lem-d2-part-5}
\end{proof}
We can now combine the above claims to prove \Cref{lem:ann-privacy}.
\begin{proof}[Proof of \Cref{lem:ann-privacy}]
Combining \Cref{claim:lem-d2-part-1}, \Cref{claim:lem-d2-part-2}, \Cref{claim:lem-d2-part-3}, \Cref{claim:lem-d2-part-4}, and \Cref{claim:lem-d2-part-5} altogether proves $e^{-c}\Pr[\cI(X) = i] \leq \Pr[\cI'(X, S) = i] \leq e^c\Pr[\cI(X) = i]$.
\end{proof}

% \begin{lemma}
% E' (\Cref{alg:lazy-gumbel-top-k-runtime}) is $(\eps + 2c)$-DP.
% \end{lemma}

We are now ready to prove \Cref{thm:ann-runtime-preserving}.

\begin{proof}[Proof of \Cref{thm:ann-runtime-preserving}]
We first prove that \Cref{alg:lazy-gumbel-top-k-runtime} is $(\eps+2c)$-DP.

Let $\Omega_S$ be the sample space of all possible approximate top-$k$ sets, and $\cI'(X)$ be the output of the index, which includes the random coins of the selection of $S\sim \Omega_S$. Then, for all $T\subseteq [n]$, we have
\begin{align*}
    \Pr[\cI'(X)\in T] &= \sum_{S\in\Omega_S}\sum_{i\in T}\Pr[S]\cdot \Pr[\cI'(X,S)=i]\\
    &\leq \sum_{S\in \Omega_S}\sum_{i\in T}\Pr[S]\cdot e^c\Pr[\cI(X)=i] \tag{by \Cref{lem:ann-privacy}}\\
    &= \sum_{S\in \Omega_S}\Pr[S]\cdot e^c\Pr[\cI(X)\in T]\\
    &= e^c\Pr[\cI(X)\in T].
\end{align*}

Next, let $X'$ be the neighboring dataset of $X$. We will prove that $\Pr[\cI'(X) \in T] \leq e^{\eps + 2c} \Pr[\cI'(X') \in T]$. We have
\begin{align*}
    \Pr[\cI'(X) \in T] &\leq e^c\Pr[\cI(X) \in T]\\
    &\leq e^{\eps + c}\Pr[\cI(X') \in T] \tag{\Cref{alg:lazy-gumbel} is $\eps$-DP}\\
    &\leq e^{\eps + 2c}\Pr[\cI'(X') \in T].
\end{align*}

Similarly, we have $e^{-\eps - 2c} \Pr[\cI'(X') \in T] \leq \Pr[\cI'(X) \in T]$ for all $T\subseteq [n]$. Combining \Cref{lem:ann-runtime}, this concludes the proof that \Cref{alg:lazy-gumbel-top-k-runtime} is $(\eps + 2c)$-DP and runs in expected time $\Theta(\sqrt n)$.
\end{proof}

\subsection{Preserving Privacy}
Here, we describe our algorithm for preserving privacy under approximation, at the cost of a slightly higher runtime.

The privacy preserving algorithm is identical to \Cref{alg:lazy-gumbel-top-k-runtime}, except we now set $B=M-m-c$.

\begin{algorithm}[H]
\caption{Lazy Gumbel Sampling using Approximate Top-$k$ while preserving runtime}
\label{alg:lazy-gumbel-top-k-privacy}
\begin{algorithmic}[1]
\State \textbf{Inputs:} An integer $k \in [n]$, and a score dataset $X = \{x_1, x_2, x_3, ..., x_n\} \in \mathbb{R} ^ n$
\vspace{1mm}
\State Let $S$ be an \textbf{approximate top-$k$} of $X$
\For{$j \in S$}
    \State Sample $G_j \sim \mathrm{Gumbel}(0, 1)$
\EndFor
\State Let $M = \max_{j \in S}(x_j + G_j)$, $m = \min_{j \in S} x_j$, and $B = M - m - c$
\State Let $C \sim \mathrm{Bin}(n - k, 1 - \exp(-\exp(-B)))$ %be the number of indices in $[n]\setminus S$ that receives a Gumbel noise at least $B$.
\State {Let $T\subseteq [n]\setminus S$ be the set of $C$ distinct indices sampled from $[n]\setminus S$.}% Sample $C$ indices from $[n]\setminus S$ as a set $T$.
\State Sample $U_i\sim \mathrm{Uniform}(\exp(-\exp(-B)), 1)$, and let $G_i=-\ln(-\ln(U_i))$ for $\forall i\in T$. 

% \State Sample $G_i \sim \mathrm{Gumbel}(0, 1)$, conditioned on $G_i >B$ for $\forall i \in T$. 
\vspace{1mm}
\State \textbf{return} $\arg\max_{i \in S \cup T}(x_i + G_i)$
\end{algorithmic}
\end{algorithm}

% First, we will show that \Cref{alg:lazy-gumbel-top-k-privacy} is $\eps$-DP.
\begin{theorem}\label{thm:ann-privacy}
    Given the set $S$, the approximate top-$k$ scores of dataset $X$, \Cref{alg:lazy-gumbel-top-k-privacy} runs in expected time $e^c\cdot\Theta(\sqrt n)$ and is $\eps$-DP.
\end{theorem}
\begin{proof}
We first prove that \Cref{alg:lazy-gumbel-top-k-privacy} is $\eps$-DP.

Given our approximate top-$k$ set $S$, we have that $\max_{i \notin S} x_i - \min_{i \in S} x_i \leq c$. For an element $j \notin S$ to achieve $j=\arg\max_{i \in [n]}\{x_i + G_i\}$, the Gumbel noise $G_j$ need to be at least 
\begin{align*}
    \max_{i \in S}\{x_i + G_i\} - x_j &\geq \max_{i \in S}\{x_i + G_i\} - \max_{i \notin S} \{x_i\}\\
    &\geq \max_{i \in S}\{x_i + G_i\} - \min_{i \in S} \{x_i\}- c\\
    &= M - m - c\\
    &= B.
\end{align*}
Therefore, we guarantee that if $j \notin S$ and $G_j< B$, index $j$ cannot be chosen. Thus, \Cref{alg:lazy-gumbel-top-k-privacy} follows the same distribution as exponential mechanism and has the same privacy guarantees, which is $\eps$-DP.

It is also proven by Mussmann et al. \citep{mussmann2017fast} that $\mathbb{E}[C] \leq \frac{e^cn}{k}$. Therefore, when $k =\Theta(\sqrt{n})$, we need to sample $\sqrt{n}$ maximum score from indices $[n]$ and at most $\frac{e^cn}{k} = e^c\sqrt{n}$ extra indices from $[n]\setminus S$ in expectation. Thus, we need to sample $e^c\cdot\Theta(\sqrt{n})$ elements in expectation.

Thus, \Cref{alg:lazy-gumbel-top-k-privacy} runs in expected time $e^c\cdot \Theta(\sqrt n)$ and is $\eps$-DP.
\end{proof}

\section{Proof of \Cref{thm:dense-mwu-fast}}
\label{sec:proof-of-dense-mwu-fast}
We rely on the following lemma by \cite{hsu2014privately}:
\begin{lemma}
    \label{lemma:non-private}
    Let $0 \leq \alpha \leq 9\rho$ and $\beta \in (0,1)$. With probability at least $1-\beta$, dense MWU with density $s$ and equipped with a $(\frac{\alpha}{3},\frac{\beta}{T})$--approximate, $\rho$-bounded dual oracle finds some $x^* \in \mathcal{K}_\text{OPT}$ for which $A_i x^* \leq b_i + \alpha$ except for at most $s-1$ constraints. 
\end{lemma}
Our algorithm effectively implements the exponential mechanism. Therefore, our dual oracle has width:
$$
\rho = \frac{\text{OPT}}{c_{\min}}-b_{max}
$$
where $c_{\min} := \min_{i=1}^d c_i$ and $b_{\max} = \max_{i=1}^d b_i$. Recall the quality scores $Q(j,y) = -\frac{\text{OPT}}{c_j}y^T A_{:,j}$. The sensitivity is at most $\frac{3\text{OPT}}{c_{\min}s}$ because $y$ can change up to $\frac{2}{s}$ (\Cref{lemma:sensitivity-bregman}) and a new constraint may also be added in the sum. As a result, by the guarantees of the exponential mechanism, we retrieve the optimal point with error at most
$$
\alpha = \frac{6\text{OPT}}{c_{\min}s\varepsilon'}\log d\log\frac{T}{\beta},\quad\text{   where    }\, \varepsilon' = \frac{\varepsilon}{\sqrt{2T\log(1/\delta)}}.
$$
Expanding, we get
$$
s = O\left(\frac{\text{OPT}^2\log d\log^{1/2}m\log\frac{T}{\beta}\log^{1/2}(1/\delta)}{c^2\alpha^2\varepsilon}\right).
$$

\section{Index Configuration}
\label{sec:index-config}
We compare three FAISS index types for approximate nearest neighbor search:
\begin{itemize}
    \item The \textbf{Flat} index serves as the baseline, performing exact inner product search via linear scan with $O(m)$ complexity per query. 
    \item For \textbf{IVF} (Inverted File Index), we partition the query vectors into $\texttt{nlist} = \max(2\sqrt{m}, 20)$ Voronoi cells using $k$-means clustering, and search the $\texttt{nprobe} = \min(\texttt{nlist}/4, 10)$ nearest cells at query time. This reduces the search space from $m$ to approximately $m \cdot \texttt{nprobe}/\texttt{nlist}$ vectors on average. 
    \item For \textbf{HNSW} (Hierarchical Navigable Small World), we construct a proximity graph with $M = 32$ neighbors per node, using $\texttt{efConstruction} = 100$ during index building and $\texttt{efSearch} = 64$ during querying. HNSW achieves approximately $O(\log m)$ query complexity by navigating through a hierarchical graph structure. 
\end{itemize} 

Both approximate indices trade off some accuracy for improved query time, with IVF offering faster index construction and HNSW providing better query performance at scale.

\section{Additional Experiments on Linear Queries}
\label{sec:more-experiments-linear-queries}
In this section we provide additional experimental results and ablation studies for the application of Fast MWEM on the linear query release problem. Since our experiments in this section are not related to performance or efficiency, we use a flat index throughout.

\subsection{Analyzing the margin $B$} 
Recall that in Fast MWEM, apart from the top $\sqrt{m}$ scores $S$ we also need to sample $C$ scores outside of $S$, where $C \sim \text{Bin}(n-\sqrt{m}, 1-\exp(-\exp(-B)))$. Here $B$ is a margin we calculate based on the maximum pertrubed score and the minimum score in $S$. If $B$ is very large, we risk having to collect a large additional sample, effectively losing the sublinear runtime guarantee. 

\citet{mussmann2017fast} show that in expectation we have $C = O(\sqrt{m})$. We ran an additional experiment to verify that this is the case. Over $T=500$ iterations and for $m \in \{500, 2000, 20000\}$ we calculate $C$ and the percentage of $m$ that we need to take additional samples of. \textbf{We confirm that the fraction of additional samples we take is a very small fraction of $m$, in the order of $O(\frac{1}{\sqrt{m}})$}.

\begin{figure}
    \centering
    \includegraphics[width=0.85\linewidth]{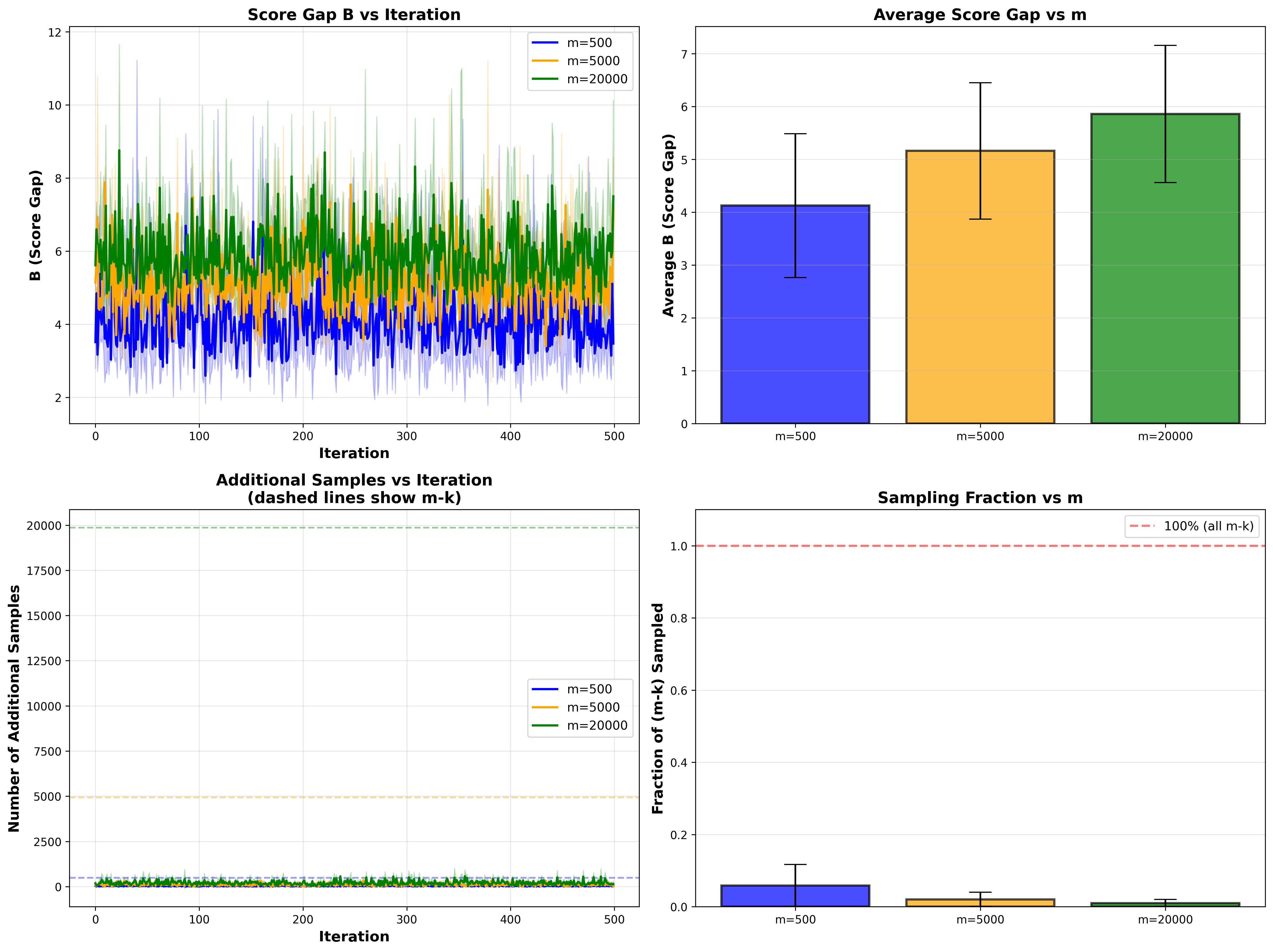}
    \caption{Analyzing the margin $B$}
    \label{fig:b-analysis}
\end{figure}

\subsection{Ablation with number of samples}
We also verify that the number of samples $n$ is indeed related to the error of our mechanism, in the same way it is for MWEM. Setting $m=100$ and letting $T=n^2$, \textbf{our experiment confirms the fact that more samples lead to increased accuracy}. Both MWEM and Fast MWEM behave very similarly in this experiment, which further confirms the closeness in the algorithms' behavior.

\begin{figure}
    \centering
    \includegraphics[width=0.7\linewidth]{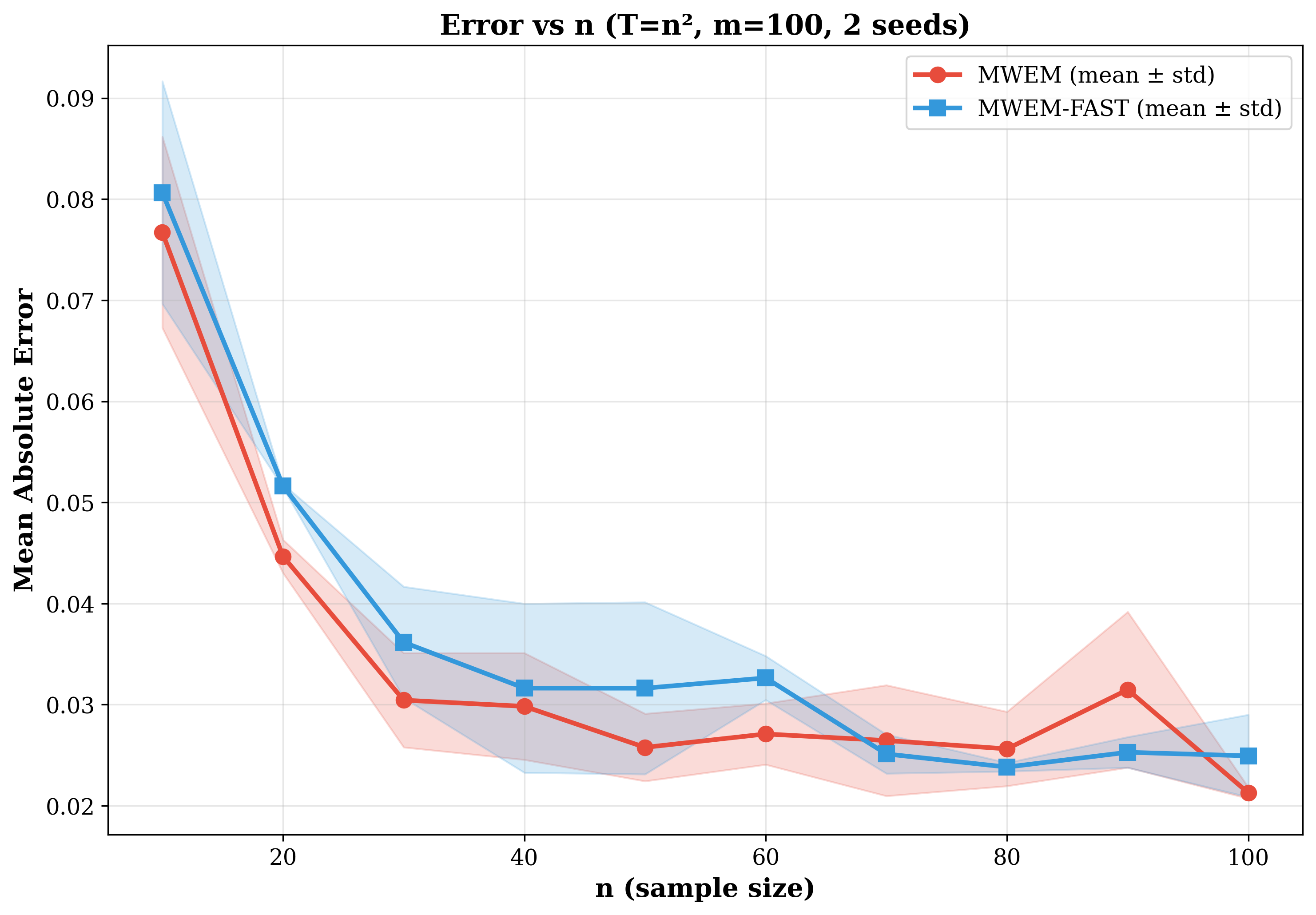}
    \caption{Ablation of the final error for different values of $n$}
    \label{fig:n-ablation}
\end{figure}

\section{Additional Plots on Linear Programming with Fast MWEM}
\label{sec:more-experiments-lps}
We also provide additional plots on our experiments with using fast MWEM for solving LPs.
\begin{itemize}
    \item As \Cref{fig:lp-appendix-index-comp} shows, both IVF and HNSW perform almost identical iterations as the original MWEM (No Index). Notably the performance of HNSW in total runtime is superior to the other approaches.
    \item When $m$ is ablated with very large values ($3\cdot 10^5 \leq m \leq 15 \times 10^5$), the runtime of HNSW manifests the sublinear speed-up by completely outperforming the other approaches, even though the build time of the data structure is larger. Unfortunately, IVF did not give us signficant speed up in these experiments. 
    
\begin{figure}[t]
    \centering
    \includegraphics[width=0.83\linewidth]{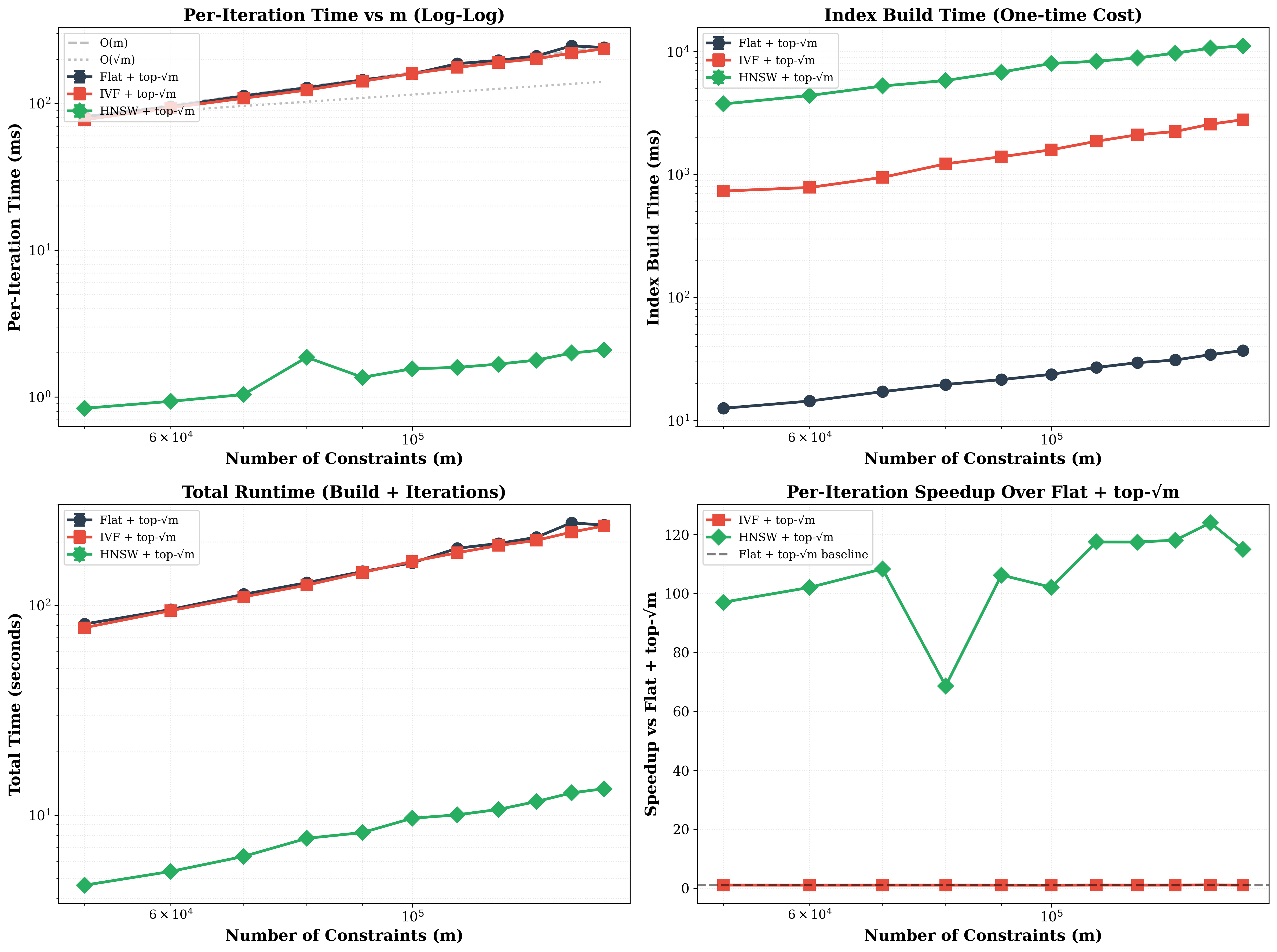}
    \caption{Comparing index performance for large ablated values of $m$}
    \label{fig:lp_index_comp}
\end{figure}
\end{itemize}

\begin{figure}
    \centering
    \includegraphics[width=0.83\linewidth]{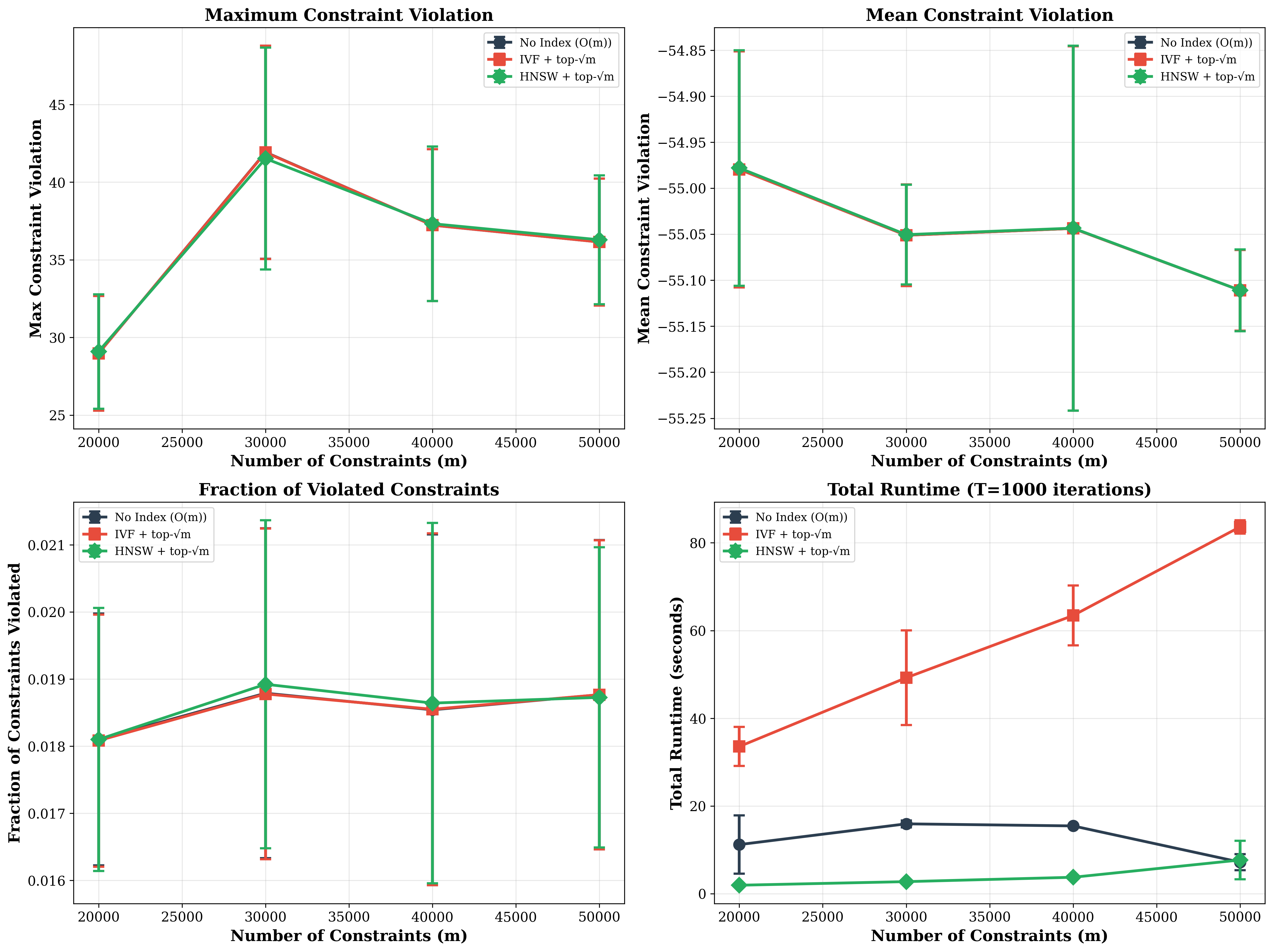}
    \caption{Comparing the error and max-constraint violations for solving scalar-private LPs with fast MWEM.}
    \label{fig:lp-appendix-index-comp}
\end{figure}

%%%%%%%%%%%%%%%%%%%%%%%%%%%%%%%%%%%%%%%%%%%%%%%%%%%%%%%%%%%%%%%%%%%%%%%%%%%%%%%
%%%%%%%%%%%%%%%%%%%%%%%%%%%%%%%%%%%%%%%%%%%%%%%%%%%%%%%%%%%%%%%%%%%%%%%%%%%%%%%

\end{document}